\documentclass{article}     
\usepackage{setspace}

\usepackage[tight,footnotesize]{subfigure}

\usepackage{amsmath} \usepackage{amsfonts} \usepackage{amssymb}
\usepackage{amsthm}
\usepackage{url}
\usepackage{natbib}
\usepackage{latexsym}
\usepackage{graphicx}

\usepackage{enumerate}
\usepackage{mdwlist} 
\usepackage{subfigure}
\usepackage{color}

\DeclareMathOperator*{\argmin}{argmin}
\newcommand{\ol}[1]{\overline{#1}}

\newcommand{\dualparvect}{\bm{a}}
\newcommand{\dualpar}{a}
\newcommand{\hyperplane}{\bm{w}}

\newcommand{\hypothesis}{h}

\newcommand{\funspace}{\mathcal{H}}

\newcommand{\kernelf}{K} 
\newcommand{\kernelm}{\bm{K}} 

\newcommand{\lossfunction}{\mathcal{L}}

\newcommand{\regparam}{\lambda}
\newcommand{\nodespace}{\mathcal{V}}
\newcommand{\nodeset}{V}
\newcommand{\edgeset}{E}

\newcommand{\geninspace}{\mathcal{X}}
\newcommand{\genex}{x}

\newcommand{\weightfunc}{Q}
\newcommand{\edgeweight}{y}

\newcommand{\node}{v}
\newcommand{\edge}{e}
\newcommand{\nodecount}{p}
\newcommand{\edgecount}{q}
\newcommand{\tset}{T}

\newcommand{\mbr}{\mathbb{R}}

\newcommand{\bm}[1]{\mathbf{#1}}

\newcommand{\idmatrix}{\bm{I}}
\newcommand{\evecmatrix}{\bm{V}}
\newcommand{\evalmatrix}{\bm{\Lambda}}
\newcommand{\transpose}{^\textnormal{T}}
\newcommand{\anymatrix}{\bm{M}}
\newcommand{\othermatrix}{\bm{N}}
\newcommand{\shufflem}{\bm{P}}
\newcommand{\ve}{\textnormal{vec}}
\newcommand{\symm}{\bm{S}}
\newcommand{\asymm}{\bm{A}}
\newcommand{\anyvector}{\bm{v}}

\newcommand{\rsize}{r}

\newcommand{\labelvector}{\bm{y}}
\newcommand{\thirdmatrix}{\bm{U}}
\newcommand{\natbase}{\bm{e}}
\newcommand{\laplacian}{\bm{L}}
\newcommand{\centerm}{\bm{C}}
\newcommand{\labelmatrix}{\bm{Y}}

\newcommand{\bookkeepmat}{\bm{B}}
\newcommand{\slrmatrix}{\bm{Q}}

\newtheorem{theorem}{Theorem}[section]
\newtheorem{lemma}[theorem]{Lemma}
\newtheorem{proposition}[theorem]{Proposition}
\newtheorem{corollary}[theorem]{Corollary}
\newtheorem{definition}[theorem]{Definition}

\newcommand{\nodedegree}{l}
\newcommand{\dimone}{s}
\newcommand{\dimtwo}{t}

\newcommand{\predfun}{f}

\begin{document}


\title{Efficient Regularized Least-Squares Algorithms for Conditional Ranking on Relational Data}


%

\author{Tapio Pahikkala and Antti Airola \\
              Department of Information Technology\\
              University of Turku, FI-20014, Turku, Finland
              \\ firstname.lastname@utu.fi   \\
           Michiel Stock, Bernard De Baets, and Willem Waegeman\\
              Department of Mathematical Modelling,\\
               Statistics and Bioinformatics,\\ Ghent University, Coupure links 653, 9000 Ghent, Belgium
              \\ firstname.lastname@ugent.be
}
\date{}

\maketitle

\begin{abstract}
In domains like bioinformatics, information retrieval and social network analysis, one can find learning tasks where the goal consists of inferring a ranking of objects, conditioned on a particular target object. We present a general kernel framework for learning conditional rankings from various types of relational data, where rankings can be conditioned on unseen data objects. We propose efficient algorithms for conditional ranking by optimizing squared regression and ranking loss functions. We show theoretically, that learning with the ranking loss is likely to generalize better than with the regression loss. Further, we prove that symmetry or reciprocity properties of relations can be efficiently enforced in the learned models. Experiments on synthetic and real-world data illustrate that the proposed methods deliver state-of-the-art performance in terms of predictive power and computational efficiency. Moreover, we also show empirically that incorporating symmetry or reciprocity properties can improve the generalization performance.
\end{abstract}

\section{Introduction}

We first motivate the study by presenting some examples relevant for the considered learning setting in Section~\ref{bgsection}. Next, we briefly review and compartmentalize related work in Section~\ref{rwsection}, and present the main contributions of the paper in Section~\ref{cosection}.

\subsection{Background}\label{bgsection}


Let us start with two introductory examples to explain the problem setting of conditional ranking. First, suppose that a number of persons are playing an online computer game. For many people it is always more fun to play against someone with similar skills, so players might be interested in receiving a ranking of other players, ranging from extremely difficult to beat to unexperienced players. Unfortunately, pairwise strategies of players in many games -- not only in computer games but also in board or sports games -- tend to exhibit a rock-paper-scissors type of relationship \citep{Fisher2008}, in the sense that player A beats player B (with probability greater than 0.5), who on his term beats C (with probability greater than 0.5), while player A loses from player C (as well with probability greater than 0.5). Mathematically speaking, the relation between players is not transitive, leading to a cyclic relationship and implying that no global (consistent) ranking of skills exists. Yet, a conditional ranking can always be obtained for a specific player \citep{Pahikkala2010}.

As a second introductory example, let us consider the supervised inference of biological networks, like protein-protein interaction networks, where the goal usually consists of predicting new interactions from a set of highly-confident interactions \citep{Yamanishi2004}. Similarly, one can also define a conditional ranking task in such a context, as predicting a ranking of all proteins in the network that are likely to interact with a given target protein \citep{Weston2004}. However, this conditional ranking task differs from the previous one because (a) rankings are computed from symmetric relations instead of reciprocal ones and (b) the values of the relations are here usually not continuous but discrete.

Applications for conditional ranking tasks arise in many domains where relational information between objects is observed, such as relations between persons in preference modelling, social network analysis and game theory, links between database objects, documents, websites, or images in information retrieval \citep{Geerts2004, grangier2008imagerank, Ng2011}, interactions between genes or proteins in bioinformatics, graph matching \citep{caetano2009matching}, et cetera. When approaching conditional ranking from a graph inference point of view, the goal consists of returning a ranking of all nodes given a particular target node, in which the nodes provide information in terms of features and edges in terms of labels or relations. At least two properties of graphs play a key role in such a setting. First, the type of information stored in the edges defines the learning task: binary-valued edge labels lead to bipartite ranking tasks \citep{Freund2003}, ordinal-valued edge labels to multipartite or layered ranking tasks \citep{Furnkranz2009} and continuous labels result in rankings that are nothing more than total orders (when no ties occur). Second, the relations that are represented by the edges might have interesting properties, namely symmetry or reciprocity, for which conditional ranking can be interpreted differently.

\begin{figure*}[t]
\begin{minipage}[c]{0.30\textwidth}
\begin{center}
\includegraphics[scale=0.25]{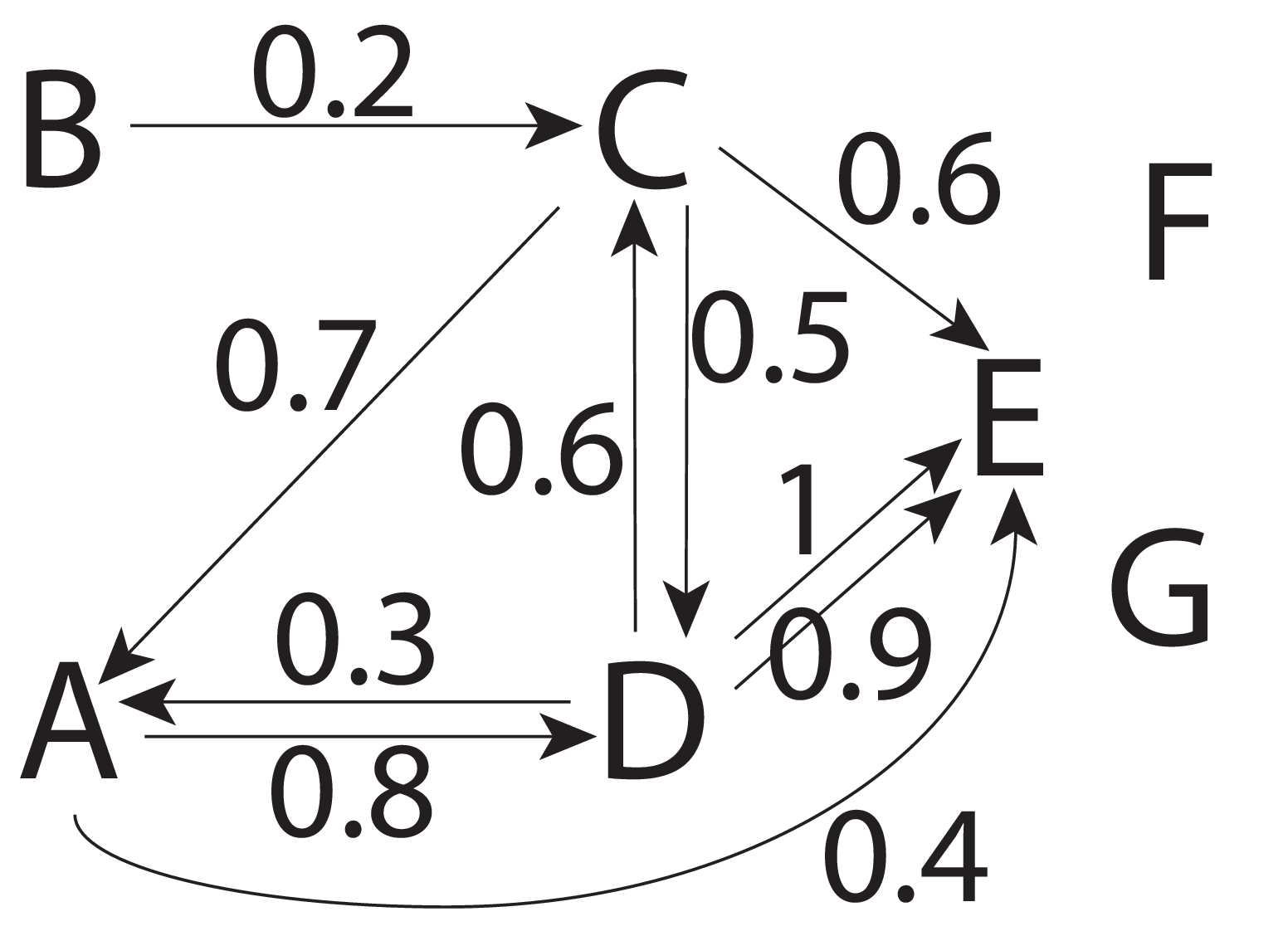}
\end{center}

\end{minipage}
\begin{minipage}[c]{0.68\textwidth}

\centering
\subfigure[C, R, T]{
\label{fig:gr:000} 
\includegraphics[scale=0.18]{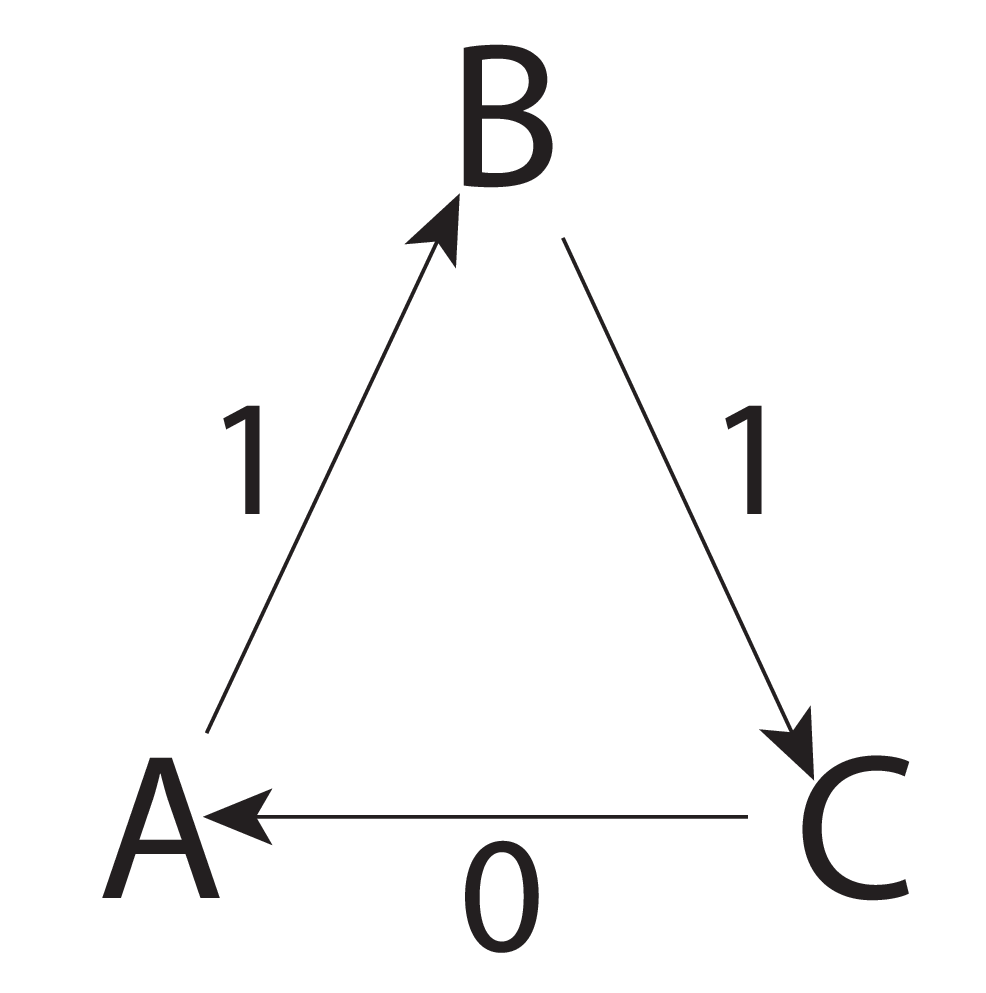}}
\subfigure[C, R, I]{
\label{fig:gr:001} 
\includegraphics[scale=0.18]{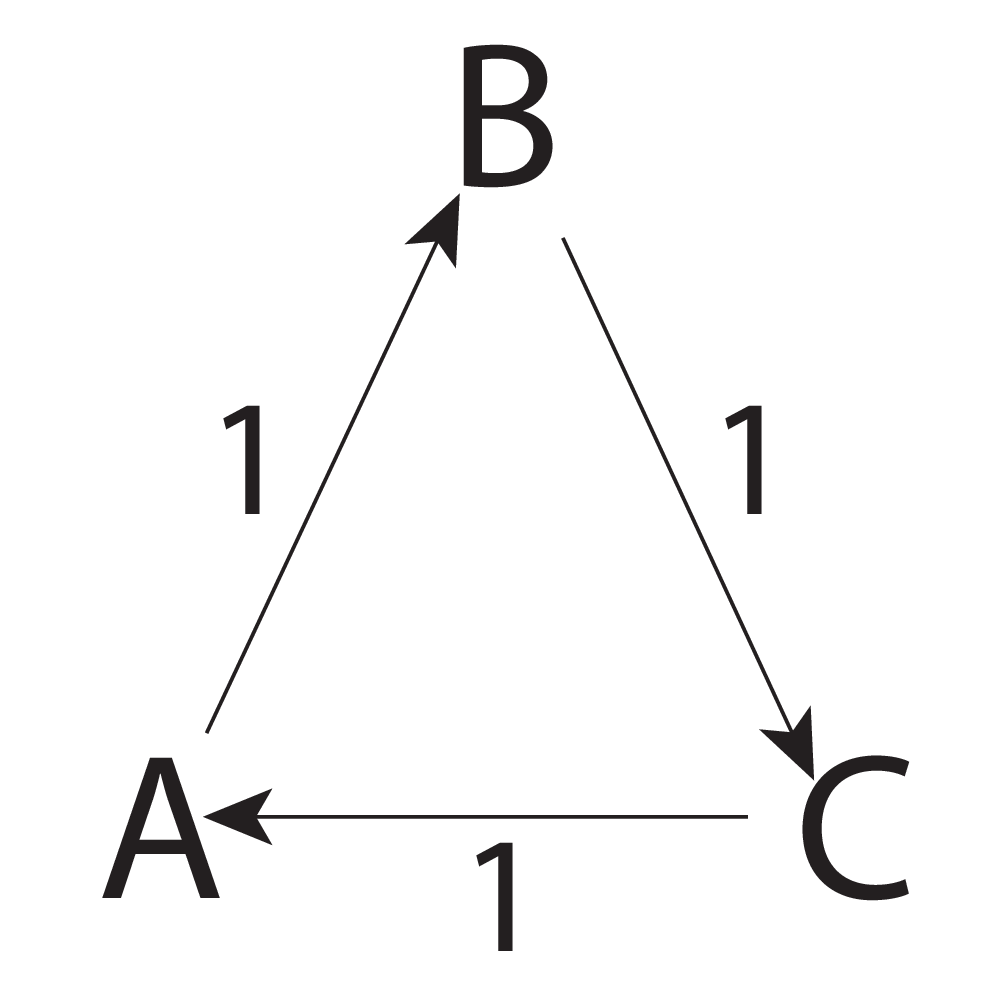}}
\subfigure[C, S, T]{
\label{fig:gr:010} 
\includegraphics[scale=0.18]{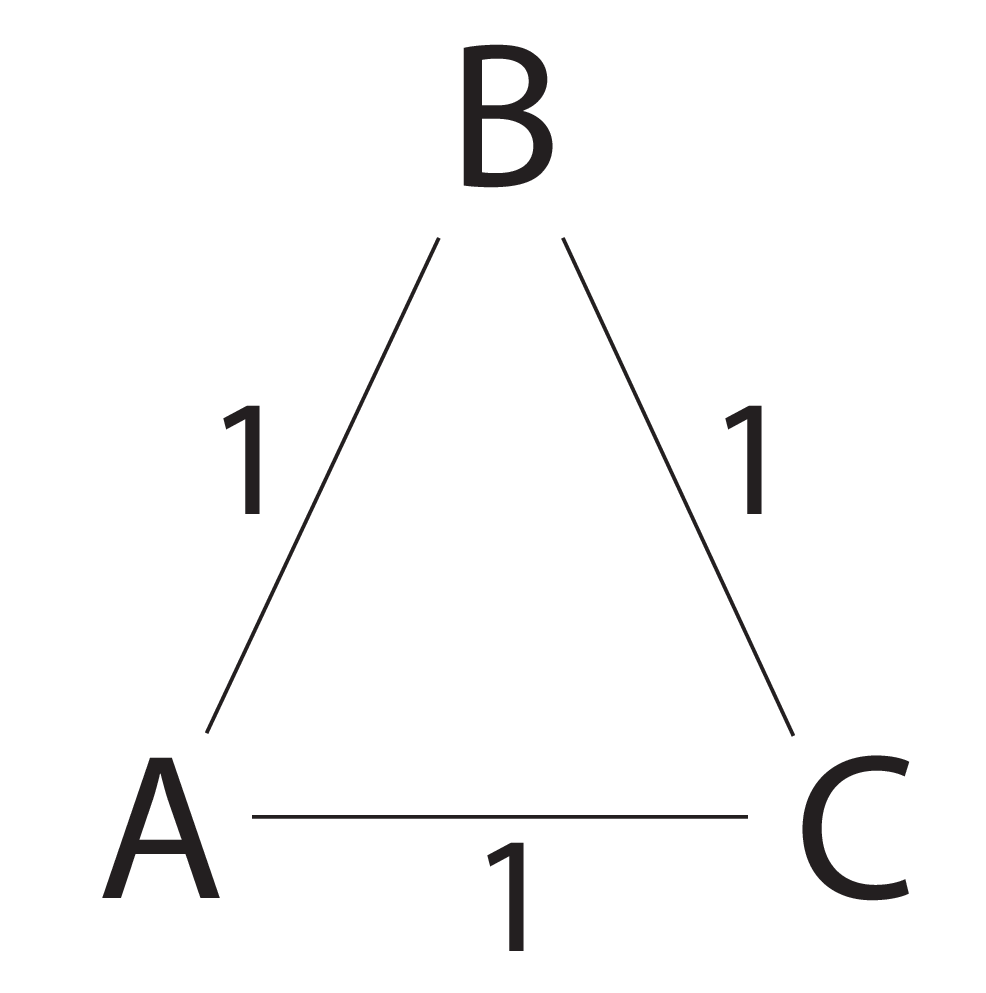}}
\subfigure[C, S, I]{
\label{fig:gr:011} 
\includegraphics[scale=0.18]{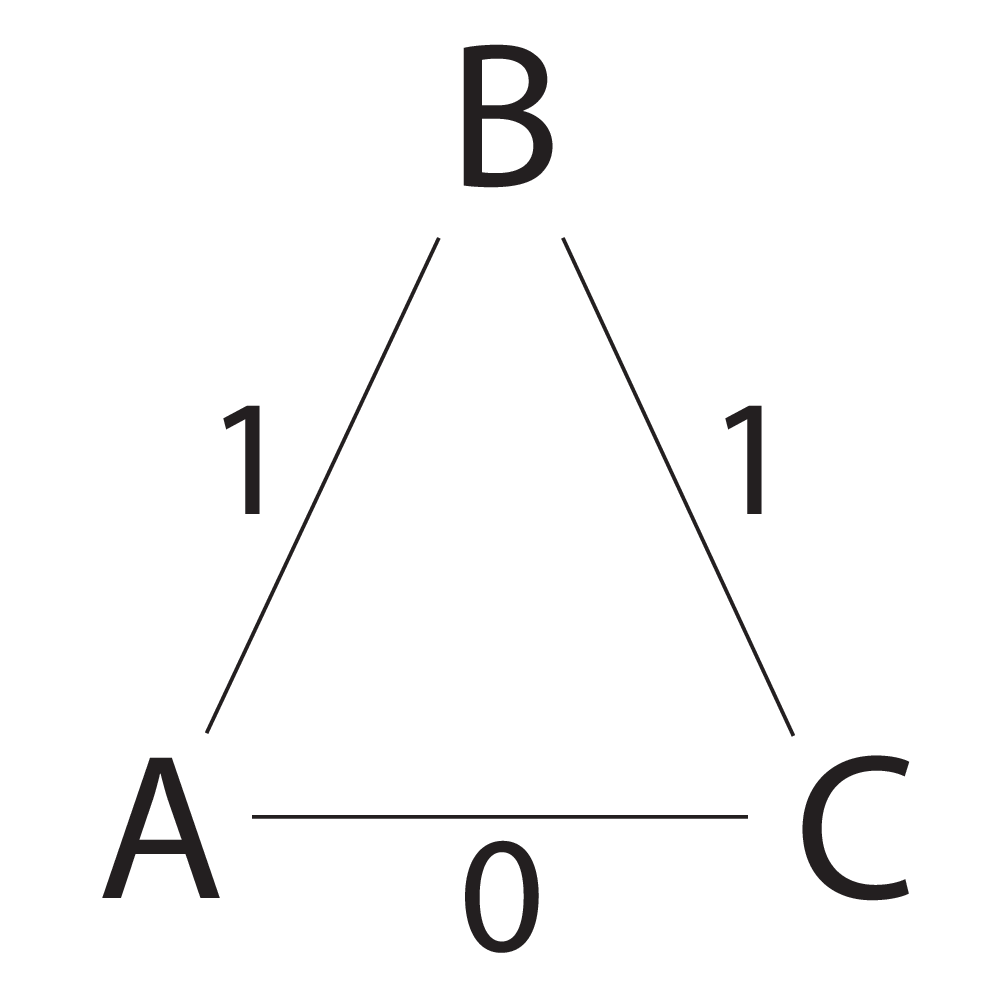}}\\
\subfigure[V, R, T]{
\label{fig:gr:100} 
\includegraphics[scale=0.18]{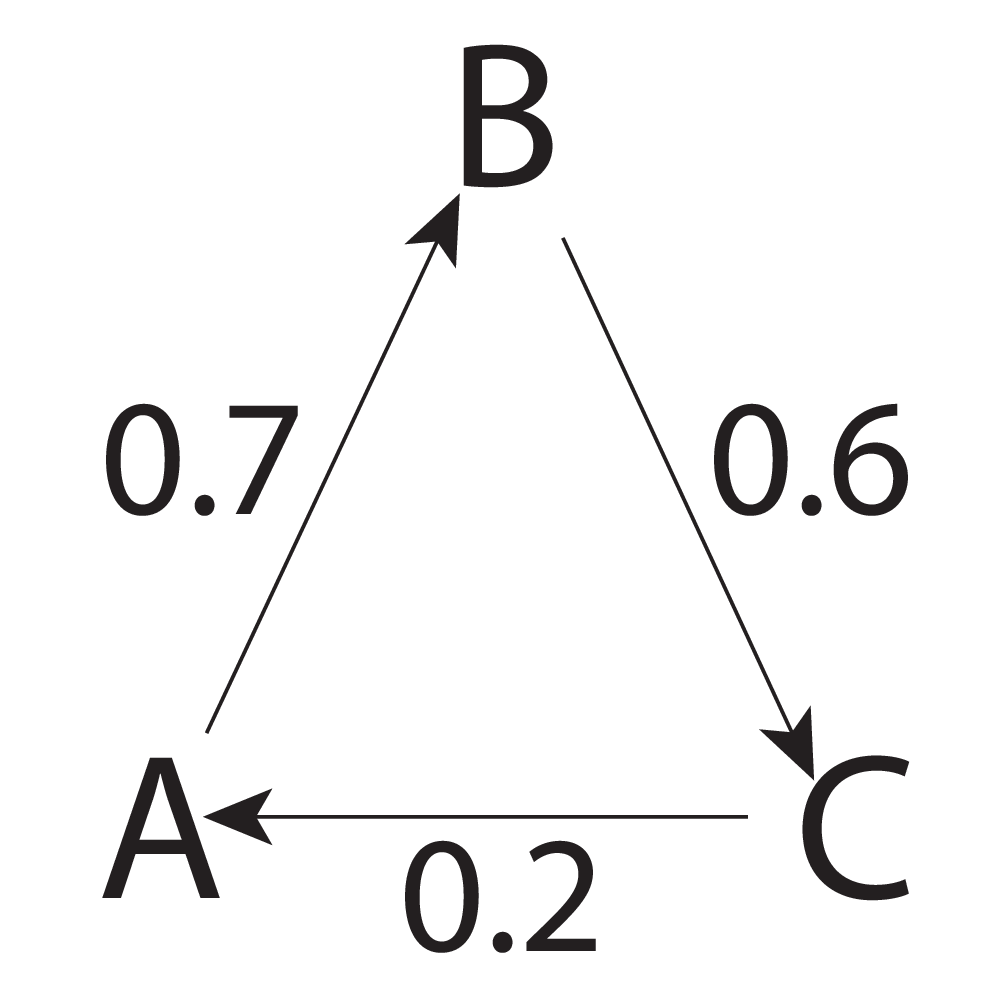}}
\subfigure[V, R, I]{
\label{fig:gr:101} 
\includegraphics[scale=0.18]{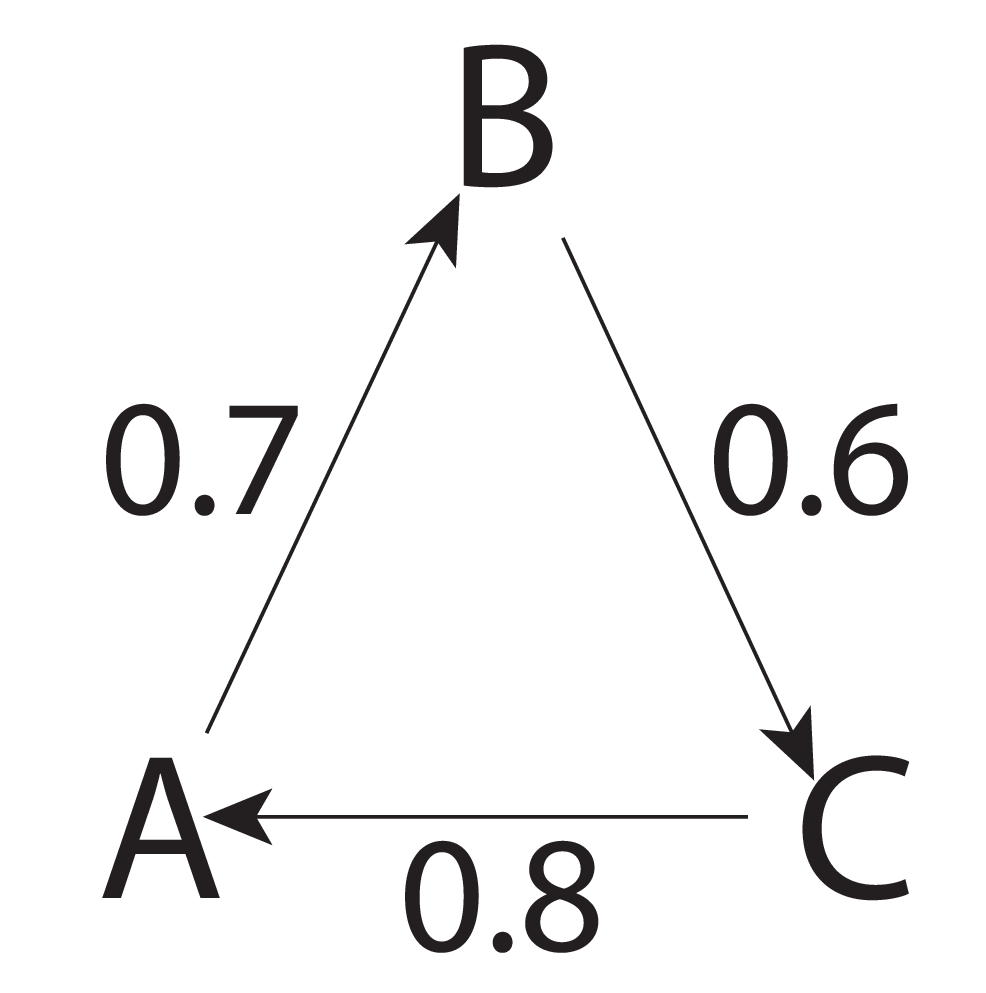}}
\subfigure[V, S, T]{
\label{fig:gr:110} 
\includegraphics[scale=0.18]{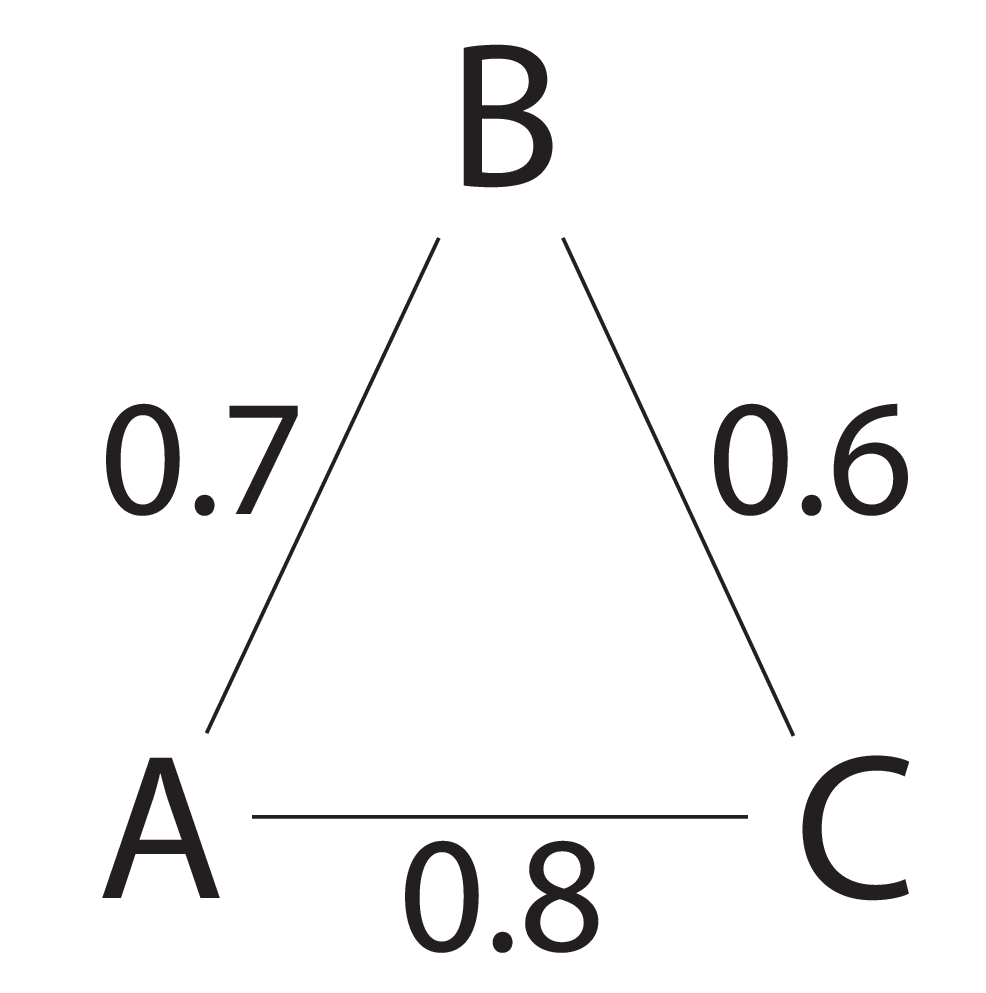}}
\subfigure[V, S, I]{
\label{fig:gr:111} 
\includegraphics[scale=0.18]{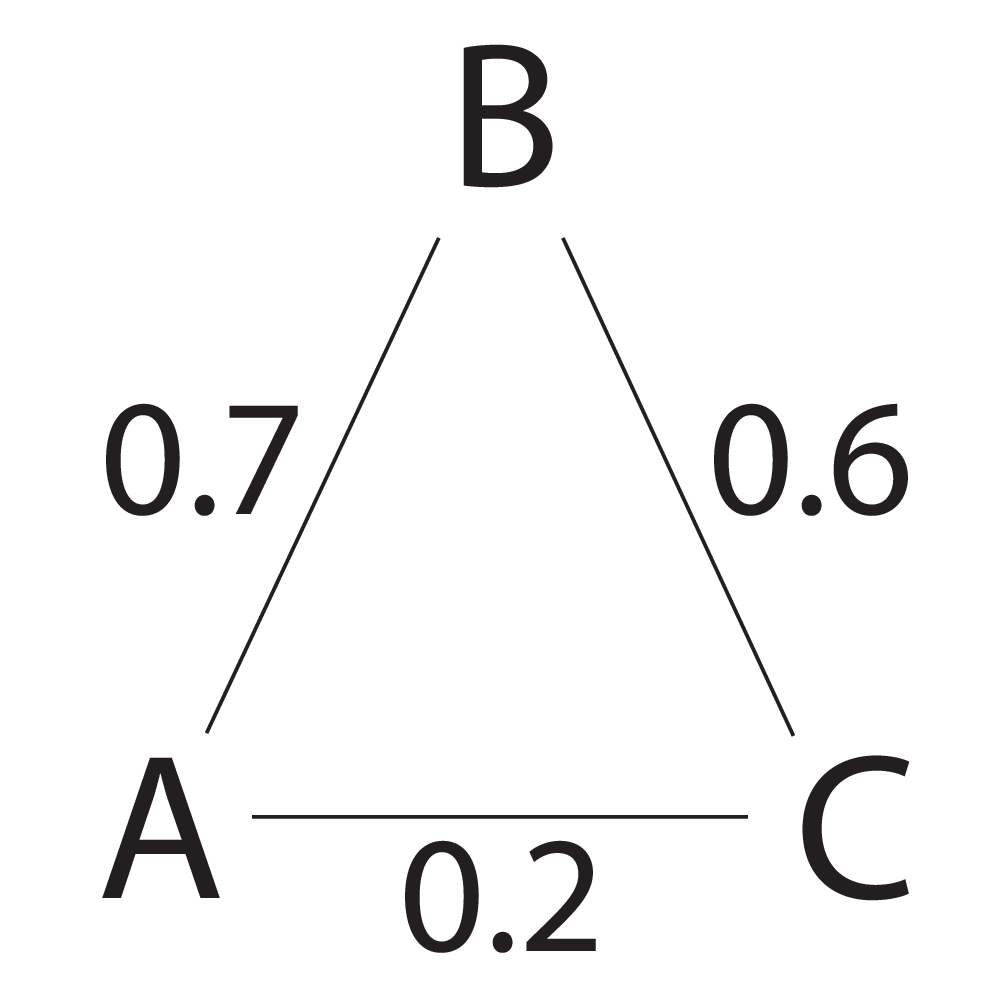}}
\end{minipage}
\caption{Left: example of a multi-graph representing the most general case where no additional properties of relations are assumed. Right: examples of eight different types of relations in a graph of cardinality three. The following relational properties are illustrated: (C) crisp, (V) valued, (R) reciprocal, (S) symmetric, (T) transitive and (I) intransitive.  For the reciprocal relations, (I) refers to a relation that does not satisfy weak stochastic transitivity, while (T) is showing an example of a relation fulfilling strong stochastic transitivity.  For the symmetric relations, (I) refers to a relation that does not satisfy $T$-transitivity w.r.t.\ the  \L ukasiewicz t-norm $T_{\bf L}(a,b) = \max(a+b-1,0)$, while (T) is showing an example of a relation that fulfills $T$-transitivity w.r.t.\ the product t-norm $T_{\bf P}(a,b) = ab$ -- see e.g. \citet{Luce1965,DeBaets2006} for formal definitions.}
\label{fig:examples}
\end{figure*}

\subsection{Learning Setting and Related Work}\label{rwsection}

We present a kernel framework for conditional ranking, which covers all above situations. Unlike existing single-task or multi-task ranking algorithms, where the conditioning is respectively ignored or only happening for training objects, our approach also allows to condition on new data objects that are not known during the training phase. Thus, in light of Figure~\ref{fig:examples} that will be explained below, the algorithm is not only able to predict conditional rankings for objects A to E, but also for objects F and G that do not participate in the training dataset. From this perspective, one can define four different learning settings in total:
\begin{itemize*}
\item Setting 1: predict a ranking of objects for a given conditioning object, where both the objects to be ranked and the conditioning object were contained in the training dataset (but not the ranking of the objects for that particular conditioning object).
\item Setting 2: given a new conditioning object unseen in the training phase, predict a ranking of the objects encountered in the training phase.
\item Setting 3: given a set of new objects unseen in the training phase, predict rankings of those objects with respect to the conditioning objects encountered in the training phase.
\item Setting 4: predict a ranking of objects for a given conditioning object, where neither the conditioning object nor the objects to be ranked were observed during the training phase.
\end{itemize*}
These four settings cover as special cases different types of conventional
machine learning problems. The framework that we propose in this article can be used for all four settings, in contrast to many existing methods. In this paper, we focus mainly on setting~4.

Setting~1 corresponds to the imputation type of link prediction setting, where missing relational values between known objects are predicted. In this setting matrix factorization methods (see e.g. \citet{Srebro2005}) are often applied. Many of the approaches are based solely on exploiting the available link structure, though approaches incorporating feature information have also been proposed \citep{Menon2010,Raymond2010scalable}. Setting~2 corresponds to the label ranking problem \citep{hullermeier2008label}, where a fixed set of labels is ranked given a new object. Setting~3 can be modeled as a multi-task ranking problem where the the number of ranking tasks is fixed in advance \citep{Agarwal2006}. Finally, setting~4 requires that the used methods are able to generalize both over new conditioning objects and objects to be ranked (see e.g. \citet{park2012flaws} for some recent discussion about this setting). Learning in setting~4 may be realized by using joint feature representations of conditioning objects and objects to be ranked.

In its most general form the conditional ranking problem can be considered as a special case of the listwise ranking problem, encountered especially in the learning to rank for information retrieval literature (see e.g. \citep{cao2006adapting,yue2007map,Xia2008listwise,Qin2008,liu2009learning,chapelle2009efficient,Qin2008,Kersting2009,Xu2010,airola2011ranksvm}). For example, in document retrieval one is supplied both with query objects and associated documents that are ranked according to how well they match the query. The aim is to learn a model that can generalize to new queries and documents, predicting rankings that capture well the relative degree to which each document matches the test query.

Previous learning approaches in this setting have typically been based on using
hand-crafted low-dimensional joint feature representations of query-document
pairs. In our graph-based terminology, this corresponds to having a given feature representation for edges, possibly encoding prior knowledge about the ranking task. A typical example of this kind of joint feature particularly designed for the domain of information retrieval is the well-known Okapi BM25 score, indicating how well a given query set of words matches a given set of words extracted from a text document. While this type of features are often very efficient and absolutely necessary in certain practical tasks, designing such application-specific features requires human experts and detailed information about every different problem to be solved, which may not be always available.

In contrast, we focus on a setting in which we are only given a feature representation of nodes from which the feature representations of the edges have to be constructed, that is, the learning must be performed without access to the prior information about the edges. This opens many possibilities for applications, since we are not restricted to the setting where explicit feature representations of the edges are provided. In our experiments, we present several examples of learning tasks for which our approach can be efficiently used. In addition, the focus of our work is the special case where both the conditioning objects and the objects to be ranked come from the same domain (see e.g. \citet{Weston2004,Yang2009} for a similar settings). This allows us to consider how to enforce relational properties such as symmetry and reciprocity, a subject not studied in previous ranking literature.
To summarize, the considered learning setting is the following:
\begin{enumerate}[(a)]
  \item We are given a training set of objects (nodes), which have a feature representation of their own, either an explicit one or an implicit representation obtained via a nonlinear kernel function.\label{firstpoint}
  \item We are also given observed relations between these objects (weighted edges), whose values are known only between the nodes encountered in the training set. This information is distinct from the similarities in point~(\ref{firstpoint}). \label{secondpoint}
  \item The aim is to learn to make predictions for pairs of objects (edges) for which the value of this relation is unknown, by taking advantage of the features or kernels given in point~(\ref{firstpoint}), and in the conditional ranking setting, the goal is to learn a ranking of the object pairs.
  \item The node kernel is used as a building block to construct a pairwise kernel able to capture similarities of the edges, which in turn, is used for learning to predict the edge weights considered in point~(\ref{secondpoint}).
\end{enumerate}

The proposed framework is based on the Kronecker product kernel for generating implicit joint feature representations of conditioning objects and the sets of objects to be ranked. This kernel has been proposed independently by a number of research groups for modelling pairwise inputs in different application domains \citep{basilico2004unifying,oyama2004using,Benhur2005}. From a different perspective, it has been considered in structured output prediction methods for defining joint feature representations of inputs and outputs \citep{Tsochantaridis2005,Weston2007}.

While the usefulness of Kronecker product kernels for pairwise learning has been
clearly established, computational efficiency of the resulting algorithms
remains a major challenge. Previously proposed methods require the explicit
computation of the kernel matrix over the data object pairs, hereby introducing
bottlenecks in terms of processing and memory usage, even for modest dataset
sizes. To overcome this problem, one typically applies sampling strategies to
avoid computing the whole kernel matrix for training. However, non-approximate methods can be implemented by taking advantage of the Kronecker structure of the kernel matrix. This idea has been traditionally used to solve certain linear regression problems (see e.g. \citet{Loan2000ubiquitous} and references therein). More recent and related applications have emerged in link prediction tasks by \citep{Kashima2009linkprob,Raymond2010scalable}, which can be considered under setting~1.

An alternative approach known as the Cartesian kernel has been proposed by \citet{kashima2009pairwise} for overcoming the computational challenges associated with the Kronecker product kernels. This kernel indeed exhibits interesting computational properties, but it can be solely employed in selected applications, because it cannot make predictions for (couples of) objects that are not observed in the training dataset, that is, in setting~4 (see \citet{Waegeman2011b} for further discussion and experimental results).

There exists a large body of literature about relational kernels, with values obtained from e.g. similarity graphs of data points, random walk and path kernels, et cetera. These can be considered to be complementary to the Kronecker product based pairwise kernels in the sense that they are used to infer similarities for data points rather than for pairs of data points. Thus, if relational information, such as paths or random walks for example, is used to form a kernel for the points in a graph, a pairwise kernel could be subsequently used to construct a kernel for the edges of the same graph.

Yet another family of related methods consists of the generalization of the
pairwise Kronecker kernels framework to tasks, in which the condition and target
objects come from different domains. Typical examples of this type of Kronecker
kernel applications are found in bioinformatics, such as the task of predicting
interactions between drugs and targets \citep{vanlaarhoven2011gaussian}. To our knowledge, none of these studies still concern the fourth setting. While the algorithms considered in this paper can be straightforwardly generalized to the two domain case, we only focus on the single domain case for simplicity, because the concepts of symmetric and reciprocal relations would not be meaningful with two domains.


\subsection{Contributions}\label{cosection}

The main contributions of this paper can be summarized as follows:
\begin{enumerate}[1]
  \item We propose kernel-based conditional ranking algorithms for
  setting~4, that is, for cases where predictions are performed for (couples of) objects that are not
  observed in the training dataset. We consider both regression and
  ranking based losses, extending regularized least-squares (RLS)
  \citep{saunders1998krr,evgeniou2000rnandsvm} and the RankRLS
  \citep{pahikkala2009preferences} algorithms to conditional ranking.
  We propose both update rules for iterative optimization algorithms, as well as
  closed-form solutions, exploiting the structure of the Kronecker product in
  order to make learning efficient. The proposed methods scale to
  graphs consisting of millions of labeled edges.
  \item We show how prior knowledge about the underlying relation can be
  efficiently incorporated in the learning process, considering the specific
  cases of symmetric and reciprocal relations. These properties are enforced on
  the learned models via corresponding modifications of the Kronecker product
  kernel, namely the symmetric kernel studied by \citet{Benhur2005} and the
  reciprocal kernel introduced by us \citep{Pahikkala2010}. We prove that, for
  RLS, symmetry and reciprocity can be implicitly encoded by including each
  edge in the training set two times, once for each direction. To our
  knowledge, the only related result so far has been established for the support vector machine classifiers by \citet{brunner2012pairwise}. We also prove that this implicitness, in turn, ensures the computational efficiency of the training phase with symmetric and reciprocal Kronecker kernels. These results are, to our knowledge, completely novel in the field of both machine learning and matrix algebra.
  \item We present new generalization bounds, showing the benefits of
  applying RankRLS instead of basic RLS regression in conditional ranking
  tasks. The analysis presented in this paper shows that the larger expressive power of the space of regression functions compared to the corresponding space of conditional ranking functions indicates the learning to be likely to generalize better if the space of functions available for the training algorithm is restricted to conditional ranking functions rather than all regression functions.
  \item Finally, we evaluate the proposed algorithms with an array of different
  practical problems. The results demonstrate the ability of the algorithms to
  solve learning problems in setting 4. Moreover, in our scalability
  experiments, we show that the algorithms proposed by us scale considerably
  better to large data sets than the state-of-the-art RankSVM solvers, even in
  cases where the SVMs use fast primal training methods for linear kernels.
\end{enumerate}


\subsection{Organization}

The article is organized as follows. We start in Section~\ref{frameworksection} with a formal description of conditional ranking from a graph-theoretic perspective. The Kronecker product kernel is reviewed in Section~\ref{relationsection} as a general edge kernel that allows for modelling the most general type of relations. In addition, we briefly recall two important subclasses of relations, namely symmetric and reciprocal relations, for which more specific, knowledge-based kernels can be derived. The proposed learning algorithms are presented in Section~\ref{algorithmsection}, and the connections and differences with related learning algorithms are discussed in Section~\ref{adaptingsection}, with a particular emphasis on the computational complexity of the algorithms. In Section~\ref{experimentsection} we present promising experimental results on synthetic and real-world data, illustrating the advantages of our approach in terms of predictive power and computational scalability.

\section{General Framework}\label{frameworksection}

Let us start with introducing some notations. We consider ranking of data structured as a graph $G = (\nodeset,\edgeset,\weightfunc)$, where $\nodeset\subseteq\nodespace$ corresponds to the set of nodes, where nodes are sampled from a space $\nodespace$, and $\edgeset \subseteq 2^{\nodeset^2}$ represents the set of edges~$\edge$, for which labels are provided in terms of relations. Moreover, these relations are represented by weights $\edgeweight_{\edge}$ on the edges and they are generated from an unknown underlying relation $\weightfunc: \nodespace^2 \rightarrow [0,1]$. We remark that the interval $[0,1]$ is used here only due to certain properties that are historically defined for such relations. However, relations taking values in arbitrary closed real intervals can be straightforwardly transformed to this interval with an appropriate increasing bijection and vice versa.

Following the standard notations for kernel methods, we formulate our learning problem as the selection of a suitable function $\hypothesis\in\funspace$, with $\funspace$ a certain hypothesis space, in particular a reproducing kernel Hilbert space (RKHS). Given an input space $\geninspace$ and a kernel $\kernelf:\geninspace\times\geninspace\rightarrow\mathbb{R}$, the RKHS associated with $\kernelf$ can be considered as the completion of
\begin{equation*}
\left\{ f \in \mathbb{R}^\geninspace  \left\arrowvert
f(\genex)=\sum_{i=1}^m\beta_i\kernelf(\genex,\genex_i)\right.\right\}\,,
\end{equation*}
in the norm
\[
\Arrowvert f\Arrowvert_\kernelf
=\sqrt{\sum_{i,j}\beta_i\beta_j\kernelf(\genex_i,\genex_j)}\,,
\]
where $\beta_i\in\mathbb{R},m\in\mathbb{N},\genex_i\in\geninspace$.

Hypotheses  $h: \nodespace^2 \rightarrow \mbr$ are usually denoted as $h(\edge) = \langle \bm{w}, \Phi(\edge) \rangle$ with $\bm{w}$ a vector of parameters that need to be estimated based on training data. Let us denote a training dataset of cardinality $\edgecount = |\edgeset|$ as a set
$T = \{(\edge,\edgeweight_{\edge}) \mid \edge \in \edgeset \}$
of input-label pairs, then we formally consider the following variational problem in which we select an appropriate
hypothesis $\hypothesis$ from $\funspace$ for training
data $T$. Namely, we consider an algorithm
\begin{eqnarray}\label{regalgorithm}
\mathcal{A}(T)=\argmin_{\hypothesis\in\funspace} \lossfunction(\hypothesis,T)
+\regparam\Arrowvert \hypothesis\Arrowvert_{\funspace}^2 \,
\end{eqnarray}
with $\lossfunction$ a given loss function and $\regparam>0$ a regularization parameter.

According to the representer theorem \citep{Kimeldorf1971results},
any minimizer $\hypothesis\in\funspace$ of
(\ref{regalgorithm}) admits a dual representation of the following form:
\begin{eqnarray}
\label{eq:primalmodel}
\hypothesis(\ol{\edge})
=\langle \hyperplane,\Phi(\ol{\edge})\rangle = \sum_{\edge \in \edgeset}
\dualpar_{\edge}\kernelf^{\Phi}(\edge,\ol{\edge}) \,,
\end{eqnarray}
with $\dualpar_{\edge} \in\mathbb{R}$ dual parameters, $\kernelf^{\Phi}$ the kernel function associated with the RKHS and $\Phi$ the feature mapping corresponding to $\kernelf^{\Phi}$.

Given two relations $\weightfunc(\node,\node')$ and $\weightfunc(\node,\node'')$ defined on any triplet of nodes in $\nodespace$, we compose the ranking of $\node'$ and $\node''$ conditioned on $\node$ as
\begin{eqnarray}
\label{eq:condorder}
\node' \succeq_{\node} \node'' \Leftrightarrow \weightfunc(\node,\node') \geq \weightfunc(\node,\node'') \,.
\end{eqnarray}
Let the number of correctly ranked pairs for all nodes in the dataset serve as evaluation criterion for verifying (\ref{eq:condorder}), then one aims to minimize the following empirical loss when computing the loss over all conditional rankings simultaneously:
\begin{eqnarray}
\label{eq:combloss}
\lossfunction(\hypothesis,T) =\sum_{\node\in\nodeset}\sum_{\edge,\ol{\edge}\in\edgeset_\node: \edgeweight_\edge < \edgeweight_{\ol{\edge}}} I(\hypothesis(\edge) - \hypothesis(\ol{\edge})) \,,
\end{eqnarray}
with $I$ the Heaviside function returning one when its argument is strictly positive, returning $1/2$ when its argument is exactly zero and returning zero otherwise. Importantly, $\edgeset_\node$ denotes the set of all edges starting from,
or the set of all edges ending at the node $\node$, depending on the specific task. For example, concerning the relation ``trust" in a social network, the former loss would correspond to ranking the persons in the network who \emph{are trusted by} a specific person, while the latter loss corresponds to ranking the persons who \emph{trust} that person. So, taking Figure~\ref{fig:examples} into account, we would in such an application respectively use the rankings $A \succ_C E \succ_C D$ (outgoing edges) and $D \succ_C B$ (incoming edges) as training info for node $C$.

Since (\ref{eq:combloss}) is neither convex nor differentiable, we look for an approximation that has these properties as this considerably simplifies the development of efficient algorithms for solving the learning problem.
Let us to this end start by considering the following squared loss function over the $\edgecount$ observed edges in the training set:
\begin{equation}\label{regrloss}
\lossfunction(\hypothesis,T) = \sum_{\edge\in\edgeset}(\edgeweight_\edge-\hypothesis(\edge))^2 \,.
\end{equation}
Such a setting would correspond to directly learning the labels on the edges in a regression or classification setting. For the latter case, optimizing (\ref{regrloss}) instead of the more conventional hinge loss has the advantage that the solution can be found by simply solving a system of linear equations \citep{saunders1998krr,Suykens2002,Shawetaylor2004,pahikkala2009preferences}. However, the simple squared loss might not be optimal in conditional ranking tasks. Consider for example a node $\node$ and we aim to learn to predict which of the two other nodes, $\node'$ or $\node''$, would be closer to it. Let us denote $\edge=(\node,\node')$ and $\ol{\edge}=(\node,\node'')$, and let $\edgeweight_{\edge}$ and $\edgeweight_{\ol{\edge}}$ denote the relation between $\node$ and $\node'$ and between $\node$ and $\node''$, respectively. Then it would be beneficial for the regression function to have a minimal squared difference $(\edgeweight_{\edge}-\edgeweight_{\ol{\edge}}-\hypothesis(\edge)+\hypothesis(\ol{\edge}))^2$, leading to the following loss function:
\begin{equation}\label{singlecentloss}
\lossfunction(\hypothesis,T) =\sum_{\node\in\nodeset}\sum_{\edge,\ol{\edge}\in\edgeset_\node}(\edgeweight_{\edge}-\edgeweight_{\ol{\edge}}-\hypothesis(\edge)+\hypothesis(\ol{\edge}))^2\,,
\end{equation}
which can be interpreted as a differentiable and convex approximation of (\ref{eq:combloss}).

\section{Relational Domain Knowledge}\label{relationsection}

Above, a framework was defined where kernel functions are constructed over the edges, leading to kernels of the form $\kernelf^{\Phi}(\edge,\ol{\edge})$. In this section we show how these kernels can be constructed using domain knowledge about the underlying relations. The same discussion was put forward for inferring relations. Details and formal proofs can be found in our previous work on this topic \citep{Pahikkala2010,Waegeman2011b}.

\subsection{Arbitrary Relations}

When no further restrictions on the underlying relation can be specified, the following Kronecker product feature mapping is used to express pairwise interactions between features of nodes:
\begin{eqnarray*}
\Phi(\edge) = \Phi(\node,\node') = \phi(\node) \otimes \phi(\node')\,,
\end{eqnarray*}
where $\phi$ represents the feature mapping for individual nodes and $\otimes$ denotes the Kronecker product. As shown by \citet{Benhur2005}, such a pairwise feature mapping yields the Kronecker product pairwise kernel in the dual model:
\begin{eqnarray}\label{eq:tppk}
\kernelf_{\otimes}^{\Phi}(\edge,\ol{\edge}) = \kernelf_{\otimes}^{\Phi}(\node,\node',\ol{\node},\ol{\node}')= \kernelf^{\phi}(\node,\ol{\node}) \kernelf^{\phi}(\node',\ol{\node}') \,,
\end{eqnarray}
with $\kernelf^{\phi}$ the kernel corresponding to $\phi$.

It can be formally proven that with an appropriate choice of the
node kernel $K^{\phi}$, such as the Gaussian RBF kernel, the RKHS of the
corresponding Kronecker product edge kernel $\kernelf^{\Phi}$ allows
approximating arbitrarily closely any relation that corresponds to a
continuous function from $\nodespace^2$ to $\mathbb{R}$.
Before summarizing this important result, we recollect
the definition of universal kernels.

\begin{definition} \citep{Steinwart2002consistency}
A continuous kernel $\kernelf$ on a compact metric space $\geninspace$ (i.e. $\geninspace$ is closed and bounded) is called universal if the RKHS induced by $\kernelf$ is dense in $C(\geninspace)$, where $C(\geninspace)$ is the space of all continuous functions $f : \geninspace \rightarrow \mathbb{R}$.
\end{definition}
Accordingly, the hypothesis space induced by the kernel $\kernelf$ can approximate any function in $C(\geninspace)$ arbitrarily well, and hence it has the universal approximating property.

\begin{theorem}\label{unikrontheorem} \citep{Waegeman2011b}
Let us assume that the space of nodes $\nodespace$ is a compact metric space.
If a continuous kernel $\kernelf^\phi$ is universal on $\nodespace$, then $\kernelf_{\otimes}^\Phi$ defines a universal kernel on $\nodespace^2$.
\end{theorem}
The proof is based on the so-called Stone-Weierstra{\ss} theorem (see e.g.
\citet{rudin1991functional}).
The above result is interesting because it shows, given that an appropriate loss is optimized and a universal kernel applied on the node level $K^{\phi}$, that the Kronecker product pairwise kernel has the ability to assure universal consistency, guaranteeing that the expected prediction error converges to its the lowest possible value when the amount of training data approaches infinity. We refer to \citep{Steinwart2008svmbook} for a more detailed discussion on the relationship between universal kernels and consistency. As a consequence, the Kronecker product
kernel can always be considered a valid choice for learning relations if no
specific a priori information other than a kernel for the nodes is provided about the relation that underlies the data.
However, we would like to emphasize that Theorem~\ref{unikrontheorem} does not
guarantee anything about the speed of the convergence or how large training sets are required for approximating the function closely enough. As a rule of thumb, whenever we have an access to useful prior information about the relation to be learned, it is beneficial to restrict the expressive power of the hypothesis space accordingly. The following two sections  illustrate this more in detail for two particular types of relational domain knowledge: symmetry and reciprocity.

\subsection{Symmetric Relations}
Symmetric relations form an important subclass of relations in our framework. As
a specific type of symmetric relations, similarity relations constitute the underlying relation in many application domains where relations between objects need to be learned. Symmetric relations are formally defined as follows.
\begin{definition}
A binary relation $\weightfunc: \nodespace^2 \rightarrow [0,1]$ is called a symmetric relation if for all $(\node,\node') \in \nodespace^2$ it holds that $\weightfunc(\node,\node') = \weightfunc(\node',\node)$.
\end{definition}
More generally, symmetry can be defined for real-valued relations
analogously as follows.
\begin{definition}
A binary relation $\hypothesis: \nodespace^2 \rightarrow \mbr$ is called a symmetric relation if for all $(\node,\node') \in \nodespace^2$ it holds that $\hypothesis(\node,\node') = \hypothesis(\node',\node)$.
\end{definition}
For symmetric relations, edges in multi-graphs like Figure~\ref{fig:examples} become undirected. Applications arise in many domains and metric learning or learning similarity measures can be seen as special cases. If the relation is $2$-valued as $Q: \nodespace^2 \rightarrow \{0,1\}$, then we end up with a classification setting instead of a regression setting.

Symmetry can be easily incorporated in our framework via the following modification of the Kronecker kernel:
\begin{eqnarray}\label{symmkernel}
\kernelf_{\otimes S}^{\Phi}(\edge,\ol{\edge})=\frac{1}{2} \big(\kernelf^{\phi}(\node,\ol{\node}) \kernelf^{\phi}(\node',\ol{\node}') + \kernelf^{\phi}(\node,\ol{\node}') \kernelf^{\phi}(\node',\ol{\node})\big) \,.
\end{eqnarray}
The symmetric Kronecker kernel has been previously used for predicting protein-protein interactions in bioinformatics \citep{Benhur2005}. The following theorem shows that the RKHS of the symmetric Kronecker kernel can approximate arbitrarily well any type of continuous symmetric relation.
\begin{theorem}\citep{Waegeman2011b} Let
\begin{equation*}
S(\nodespace^2)=\{t \mid t\in C(\nodespace^2),t(\node,\node')=t(\node',\node)\}
\end{equation*}
be the space of all continuous symmetric relations from $\nodespace^2$ to $\mathbb{R}$.
If $\kernelf^\phi$ on $\nodespace$ is universal, then the RKHS induced by the kernel $\kernelf_{\otimes S}^{\Phi}$ defined in (\ref{symmkernel}) is dense in $S(\nodespace^2)$.
\end{theorem}
In other words the above theorem states that using the symmetric Kronecker product kernel is a way to incorporate the prior knowledge about the symmetry of the relation to be learned by only sacrificing the unnecessary expressive power. Thus, consistency can still be assured, despite considering a smaller hypothesis space.

\subsection{Reciprocal Relations}
Let us start with a definition of this type of relation.
\begin{definition}
A binary relation $\weightfunc: \nodespace^2 \rightarrow [0,1]$ is called a reciprocal relation if for all $(\node,\node') \in \nodespace^2$ it holds that $\weightfunc(\node,\node') = 1 - \weightfunc(\node',\node)$. \end{definition}
For general real-valued relations, the notion of  antisymmetry can be used in place of reciprocity:
\begin{definition}
A binary relation $\hypothesis: \nodespace^2 \rightarrow \mbr$ is called an antisymmetric relation if for all $(\node,\node') \in \nodespace^2$ it holds that $\hypothesis(\node,\node') = - \hypothesis(\node',\node)$. \end{definition}

For reciprocal and antisymmetric relations, every edge $\edge=(\node,\node')$ in a multi-graph like Figure~\ref{fig:examples} induces an unobserved invisible edge $\edge_R = (\node',\node)$ with appropriate weight in the opposite direction. Applications arise here in domains such as preference learning, game theory and bioinformatics for representing preference relations, choice probabilities, winning probabilities, gene regulation, et cetera. The weight on the edge defines the real direction of such an edge. If the weight on the edge $\edge = (\node,\node')$ is higher than 0.5, then the direction is from $v$ to $v'$, but when the weight is lower than 0.5, then the direction should be interpreted as inverted, for example, the edges from $A$ to $C$ in Figures~\ref{fig:examples} (a) and (e) should be interpreted as edges starting from $A$ instead of $C$. If the relation is $3$-valued as $Q: \nodespace^2 \rightarrow \{0,1/2,1\}$, then we end up with a three-class ordinal regression setting instead of an ordinary regression setting. Analogously to symmetry, reciprocity can also be easily incorporated in our framework via the following modification of the Kronecker kernel:
\begin{eqnarray}
\label{eq:recedgekernel}
K_{\otimes R}^{\Phi}(\edge,\ol{\edge})=\frac{1}{2} \big(K^{\phi}(\node,\ol{\node}) K^{\phi}(\node',\ol{\node}') - K^{\phi}(\node,\ol{\node}') K^{\phi}(\node',\ol{\node})\big)\,.
\end{eqnarray}
Thus, the addition of kernels in the symmetric case becomes a subtraction of
kernels in the reciprocal case. One can also prove that the RKHS of this so-called reciprocal Kronecker kernel allows
approximating arbitrarily well any type of continuous reciprocal relation.
\begin{theorem}\label{antisymmetrictheorem} \citep{Waegeman2011b}
Let
\begin{equation*}
R(\nodespace^2)=\{ t \mid t \in C(\nodespace^2),t(\node,\node')= - t(\node',\node)\}
\end{equation*}
be the space of all continuous antisymmetric relations from $\nodespace^2$ to $\mathbb{R}$. If $\kernelf^\phi$ on $\nodespace$ is universal, then the RKHS induced by the kernel $\kernelf_{\otimes R}^{\Phi}$ defined in (\ref{eq:recedgekernel}) is dense in $R(\nodespace^2)$.
\end{theorem}
Unlike many existing kernel-based methods for relational data, the models obtained with the presented kernels are able to represent any symmetric or reciprocal relation, respectively, without imposing additional transitivity properties of the relations.

\section{Algorithmic Aspects}
\label{algorithmsection}

This section gives a detailed description of the different algorithms that we propose for conditional ranking tasks. Our algorithms are primarily based on solving specific systems of linear equations, in which domain knowledge about the underlying relations is taken into account. In addition, a detailed discussion about the differences between optimizing (\ref{regrloss}) and (\ref{singlecentloss}) is provided.

\subsection{Matrix Representation of Symmetric and Reciprocal Kernels}
Let us define the so-called commutation matrix, which provides a powerful tool for formalizing the kernel matrices corresponding to the symmetric and reciprocal kernels.
\begin{definition}[Commutation matrix]
The $\dimone^2\times\dimone^2$-matrix
\begin{equation*}
\shufflem^{\dimone^2}=\sum_{i=1}^\dimone\sum_{j=1}^\dimone\natbase_{(i-1)\dimone+j}\natbase_{(j-1)\dimone+i}\transpose
\end{equation*}
is called the commutation matrix \citep{abadir2005matrixalgebra}, where $\natbase_i$ are the standard basis vectors of $\mathbb{R}^{\dimone^2}$.
\end{definition}
We use the superscript $\dimone^2$ to indicate the dimension $\dimone^2\times\dimone^2$ of the matrix $\shufflem$ but we omit this notation when the dimensionality is clear from the context or when the considerations do not depend on the dimensionality. For $\shufflem$, we have the following properties. First, $\shufflem\shufflem=\idmatrix$, where $\idmatrix$ is the identity matrix, since $\shufflem$ is a symmetric permutation matrix. Moreover, for every square matrix $\anymatrix\in\mathbb{R}^{\dimone\times\dimone}$, we have $\shufflem\ve(\anymatrix)=\ve(\anymatrix\transpose)$, where $\ve$ is the column vectorizing operator that stacks the columns of an $\dimone\times\dimone$-matrix in an $\dimone^2$-dimensional column vector, that is,
\begin{equation}\label{vecdef}
\ve(\anymatrix)=
(\anymatrix_{1,1},
\anymatrix_{2,1},
\ldots,
\anymatrix_{\dimone,1},
\anymatrix_{1,2},
\ldots,
\anymatrix_{\dimone,\dimone})\transpose\,.
\end{equation}
Furthermore, for $\anymatrix,\othermatrix\in\mathbb{R}^{\dimone\times\dimtwo}$, we have $$\shufflem^{\dimone^2}(\anymatrix\otimes\othermatrix)=(\othermatrix\otimes\anymatrix)\shufflem^{\dimtwo^2}.$$

The commutation matrix is used as a building block in constructing the following types of matrices:
\begin{definition}[Symmetrizer and skew-symmetrizer matrices]
The matrices
\begin{eqnarray*}
\symm=\frac{1}{2}(\idmatrix+\shufflem)\textrm{ and }
\asymm=\frac{1}{2}(\idmatrix-\shufflem) \,
\end{eqnarray*}
are known as the symmetrizer and skew-symmetrizer matrix, respectively \citep{abadir2005matrixalgebra}.
\end{definition}

Armed with the above definitions, we will now consider how the kernel matrices corresponding to the reciprocal kernel $\kernelf_{\otimes R}^{\Phi}$ and the symmetric kernel $\kernelf_{\otimes S}^{\Phi}$ can be represented in a matrix notation. Note that the next proposition covers also the kernel matrices constructed between, say, nodes encountered in the training set and the nodes encountered at the prediction phase, and hence the considerations involve two different sets of nodes.
\begin{proposition}\label{symmasymmkmprop}
Let $\kernelm\in\mathbb{R}^{\rsize\times\nodecount}$ be a kernel matrix consisting of all kernel evaluations between nodes in sets $\nodeset\subseteq\nodespace$ and $\ol{\nodeset}\subseteq\nodespace$, with $\arrowvert\nodeset\arrowvert=\rsize$ and $\arrowvert\ol{\nodeset}\arrowvert=\nodecount$, that is, $\kernelm_{i,j}=\kernelf^\phi(\node_i,\ol{\node}_j)$, where $\node_i\in\nodeset$ and $\ol{\node}_j\in \ol{\nodeset}$. The ordinary, symmetric and reciprocal Kronecker kernel matrices consisting of all kernel evaluations between edges in $\nodeset\times \nodeset$ and edges in $\overline{\nodeset}\times \ol{\nodeset}$ are given by
$$\ol{\kernelm} = \kernelm\otimes\kernelm \,, \quad \ol{\kernelm}^S = \symm^{\rsize^2}(\kernelm\otimes\kernelm) \,, \quad  \ol{\kernelm}^R = \asymm^{\rsize^2}(\kernelm\otimes\kernelm) \,.$$
\end{proposition}
\begin{proof}
The claim concerning the ordinary Kronecker kernel is an immediate consequence of the definition of the Kronecker product, that is, the entries of $\ol{\kernelm}$ are given as
$$\ol{\kernelm}_{(h-1)\rsize+i,(j-1)\nodecount+k}=\kernelf^{\phi}(\node_h,\ol{\node}_j) \kernelf^{\phi}(\node_i,\ol{\node}_k)\,,$$
where $1\leq h,i\leq\rsize$ and $1\leq j,k\leq\nodecount$. To prove the other two claims, we pay closer attention to the entries of $\kernelm\otimes\kernelm$. In particular, the $((j-1)\nodecount+k)$-th column of $\kernelm\otimes\kernelm$ contains all kernel evaluations of the edges in $\nodeset\times\nodeset$ with the edge $(\ol{\node}_j,\ol{\node}_k)\in\ol{\nodeset} \times\ol{\nodeset}$. By definition (\ref{vecdef}) of $\ve$, this column can be written as $\ve(\anymatrix)$, where $\anymatrix\in\mathbb{R}^{\rsize\times\rsize}$ is a matrix whose $i,h$-th entry contains the kernel evaluation between the edges $(\node_h,\node_i)\in\nodeset\times\nodeset$ and $(\ol{\node}_j,\ol{\node}_k)\in \ol{\nodeset} \times \ol{\nodeset}$:
\[
\kernelf^{\phi}(\node_h,\ol{\node}_j) \kernelf^{\phi}(\node_i,\ol{\node}_k)\,.
\]
We have the following properties of the symmetrizer and skew-symmetrizer matrices that straightforwardly follow from those of the commutation matrix. For any $\anymatrix\in\mathbb{R}^{\dimone\times\dimone}$
{\setlength\arraycolsep{2pt}
\begin{eqnarray}
\symm\ve(\anymatrix)&=&\frac{1}{2}\ve(\anymatrix+\anymatrix\transpose)\label{symmdef}\\
\asymm\ve(\anymatrix)&=&\frac{1}{2}\ve(\anymatrix-\anymatrix\transpose)\nonumber \,.
\end{eqnarray}
}Thus, the $((j-1)\nodecount+k)$-th column of $\symm^{\rsize^2}(\kernelm\otimes\kernelm)$ can, according to (\ref{symmdef}), be written as $\frac{1}{2}\ve(\anymatrix+\anymatrix\transpose)$, where the $i,h$-th entry of $\anymatrix+\anymatrix\transpose$ contains the kernel evaluation
\[
\frac{1}{2}\big(\kernelf^{\phi}(\node_h,\ol{\node}_j) \kernelf^{\phi}(\node_i,\ol{\node}_k) + \kernelf^{\phi}(\node_i,\ol{\node}_j) \kernelf^{\phi}(\node_h,\ol{\node}_k)\big)\,,
\]
which corresponds to the symmetric Kronecker kernel between the edges $(\node_i,\node_h)\in \nodeset\times \nodeset$ and $(\ol{\node}_j,\ol{\node}_k)\in \overline{\nodeset} \times \overline{\nodeset}$. The reciprocal case is analogous. 
\end{proof}

We also note that for $\anymatrix\in\mathbb{R}^{\dimone\times\dimtwo}$ the symmetrizer and skew-symmetrizer matrices commute with the $\dimone^2\times\dimtwo^2$-matrix $\anymatrix\otimes\anymatrix$ in the following sense:
{\setlength\arraycolsep{2pt}
\begin{eqnarray}
\label{symmcomm}
\symm^{\dimone^2}(\anymatrix\otimes\anymatrix)&=&(\anymatrix\otimes\anymatrix)\symm^{\dimtwo^2}\\
\label{asymmcomm}
\asymm^{\dimone^2}(\anymatrix\otimes\anymatrix)&=&(\anymatrix\otimes\anymatrix)\asymm^{\dimtwo^2}\,,
\end{eqnarray}
}where $\symm^{\dimone^2}$ and $\asymm^{\dimone^2}$ in the left-hand sides are $\dimone^2\times\dimone^2$-matrices and $\symm^{\dimtwo^2}$ and $\asymm^{\dimtwo^2}$ are $\dimtwo^2\times\dimtwo^2$-matrices in the right-hand sides. Thus, due to (\ref{symmcomm}), the above-considered symmetric Kronecker kernel matrix may as well be written as $(\kernelm\otimes\kernelm)\symm^{\nodecount^2}$ or as $\symm^{\rsize^2}(\kernelm\otimes\kernelm)\symm^{\nodecount^2}$. The same applies to the reciprocal Kronecker kernel matrix due to (\ref{asymmcomm}).

\subsection{Regression with Symmetric and Reciprocal Kernels}

Let $\nodecount$ and $\edgecount$, respectively, represent the number of nodes and edges in $\tset$. In the following, we make an assumption that $\tset$ contains, for each ordered pair of nodes $(\node,\node')$, exactly one edge starting from $\node$ and ending to $\node'$, that is, $\edgecount=\nodecount^2$ and $\tset$ corresponds to a complete directed graph on $\nodecount$ nodes which includes a loop at each node. As we will show below, this important special case enables the use of many computational short-cuts for the training phase. This assumption is dropped in Section~\ref{conjugatesection}, where we present training algorithms for the more general case.

In practical applications, a fully connected graph is most commonly available
in settings where the edges are generated by comparing some direct property
of the nodes, such as whether they belong to same or similar class in a
classification taxonomy (for examples, see the experiments). In experimental research on small sample data,
such as commonly considered in many bioinformatics applications
(see e.g. \citep{park2012flaws}), it may also be feasible to gather the edge
information directly for all pairs through experimental comparisons.
If this is not the case, then imputation techniques can be used to fill in the
values for missing edges in the training data in case only a small minority is
missing.

Using the notation of Proposition~\ref{symmasymmkmprop}, we let $\kernelm\in\mathbb{R}^{\nodecount\times\nodecount}$ be the kernel matrix of $\kernelf^{\phi}$, containing similarities between all nodes encountered in $\tset$. Due to the above assumption and Proposition~\ref{symmasymmkmprop}, the kernel matrix containing the evaluations of the kernels $\kernelf_{\otimes}^{\Phi}$, $\kernelf_{\otimes}^{\Phi S}$ and $\kernelf_{\otimes}^{\Phi R}$ between the edges in $\tset$ can be expressed as $\ol{\kernelm}$, $\ol{\kernelm}^S$ and $\ol{\kernelm}^R$, respectively.

Recall that, according to the representer theorem, the prediction function obtained as a solution to problem (\ref{regalgorithm}) can be expressed with the dual representation (\ref{eq:primalmodel}), involving a vector of so-called dual parameters, whose dimension equals the number of edges in the training set. Here, we represent the dual solution with a vector $\dualparvect\in\mathbb{R}^{\nodecount^2}$ containing one entry per each possible edge between the vertices occurring in the training set.

Thus, using standard Tikhonov regularization \citep{evgeniou2000rnandsvm}, the objective function of problem (\ref{regalgorithm}) with kernel $\kernelf_{\otimes}^{\Phi}$ can be rewritten in matrix notation as
\begin{equation}\label{specialcaseobjfun}
\lossfunction(\labelvector,\ol{\kernelm}\dualparvect)
+\regparam\dualparvect\transpose\overline{\kernelm}\dualparvect\,,
\end{equation}
where $\lossfunction:\mathbb{R}^\edgecount\times\mathbb{R}^\edgecount\rightarrow\mathbb{R}$ is a convex loss function that maps the vector $\labelvector$ of training labels and the vector $\ol{\kernelm}\dualparvect$ of predictions to a real value.

Up to multiplication with a constant, the loss (\ref{regrloss}) can be represented in matrix form as
\begin{equation}\label{regrlossmatrixform}
(\labelvector-\ol{\kernelm}\dualparvect)\transpose(\labelvector-\ol{\kernelm}\dualparvect)\,.
\end{equation}
Thus, for the regression approach, the objective function to be minimized becomes:
\begin{equation}\label{rlsregr}
(\labelvector-\overline{\kernelm}\dualparvect)\transpose(\labelvector-\overline{\kernelm}\dualparvect)
+\regparam\dualparvect\transpose\overline{\kernelm}\dualparvect\,.
\end{equation}
By taking the derivative of (\ref{rlsregr}) with respect to $\dualparvect$, setting it to zero, and solving with respect to $\dualparvect$, we get the following system of linear equations:
\begin{equation}\label{notsimplifiedsystem}
(\overline{\kernelm}\overline{\kernelm}+\regparam\overline{\kernelm})\dualparvect=\overline{\kernelm}\labelvector\,.
\end{equation}
If the kernel matrix $\overline{\kernelm}$ is not strictly positive definite but only positive semi-definite, $\overline{\kernelm}$ should be interpreted as $\lim_{\epsilon\rightarrow 0^+}(\overline{\kernelm}+\epsilon\idmatrix)$. Accordingly, (\ref{notsimplifiedsystem}) can be simplified to
\begin{equation}\label{simplifiedsystem}
(\overline{\kernelm}+\regparam\idmatrix)\dualparvect=\labelvector\,.
\end{equation}
Due to the positive semi-definiteness of the kernel matrix, (\ref{simplifiedsystem}) always has a unique solution. Since the solution of (\ref{simplifiedsystem}) is also a solution of (\ref{notsimplifiedsystem}), it is enough to concentrate on solving (\ref{simplifiedsystem}).

\newcommand{\fourthmatrix}{\bm{V}}

Using the standard notation and rules of Kronecker product algebra, we show how to efficiently solve shifted Kronecker product systems. For a more in-depth analysis of the shifted Kronecker product systems, we refer to \citet{martin2006shiftedkron}.
\begin{proposition}\label{shiftedlemma}
Let  $\anymatrix,\othermatrix\in\mathbb{R}^{\nodecount\times\nodecount}$ be diagonalizable matrices, that is, the matrices can be eigen decomposed as
\begin{equation*}
\anymatrix=\evecmatrix\evalmatrix \evecmatrix^{-1}\,,\quad\othermatrix=\thirdmatrix\bm{\Sigma} \thirdmatrix^{-1}\,,
\end{equation*}
where $\evecmatrix,\thirdmatrix\in\mathbb{R}^{\nodecount\times\nodecount}$ contain the eigenvectors and the diagonal matrices $\evalmatrix,\bm{\Sigma}\in\mathbb{R}^{\nodecount\times\nodecount}$ contain the corresponding eigenvalues of $\anymatrix$ and $\othermatrix$. Then, the following type of shifted Kronecker product system
\begin{equation}\label{shiftedkron}
(\anymatrix\otimes\othermatrix
+\regparam\idmatrix)\dualparvect=\ve(\labelmatrix)\,,
\end{equation}
where $\regparam>0$ and $\labelmatrix\in\mathbb{R}^{\nodecount\times\nodecount}$, can be solved with respect to $\dualparvect$ in $O(\nodecount^3)$ time if the inverse of $\anymatrix\otimes\othermatrix
+\regparam\idmatrix$ exists.
\end{proposition}
\begin{proof}
Before starting the actual proof, we recall certain rules concerning the Kronecker product (see e.g. \citet{horn1991topics}) and introduce some notations. Namely, for $\anymatrix\in\mathbb{R}^{a\times b}$, $\thirdmatrix\in\mathbb{R}^{c\times d}$,
$\othermatrix\in\mathbb{R}^{b\times \dimone}$ and $\fourthmatrix\in\mathbb{R}^{d\times \dimtwo}$, we have:
\begin{equation*}
(\anymatrix\otimes\thirdmatrix)(\othermatrix\otimes \fourthmatrix)=(\anymatrix \othermatrix)\otimes(\thirdmatrix \fourthmatrix)\,.
\end{equation*}
From this, it directly follows that
\begin{equation}\label{kroninv}
(\anymatrix\otimes\othermatrix)^{-1}=\anymatrix^{-1}\otimes\othermatrix^{-1}\,.
\end{equation}
Moreover, for $\anymatrix\in\mathbb{R}^{a\times b}$, $\othermatrix\in\mathbb{R}^{b\times c}$,
and $\thirdmatrix\in\mathbb{R}^{c\times d}$, we have:
\begin{equation*}
(\thirdmatrix\transpose\otimes \anymatrix)\ve(\othermatrix)=\ve(\anymatrix\othermatrix\thirdmatrix)\,.
\end{equation*}
Furthermore, for $\anymatrix,\othermatrix\in\mathbb{R}^{a\times b}$, let $\anymatrix\odot \othermatrix$ denote the Hadamard (elementwise) product, that is, $(\anymatrix\odot \othermatrix)_{i,j}=\anymatrix_{i,j}\othermatrix_{i,j}$. Further, for a vector $\bm{v}\in\mathbb{R}^\dimone$, let $\textnormal{diag}(\bm{v})$ denote the diagonal $\dimone\times\dimone$-matrix, whose diagonal entries are given as $\textnormal{diag}(\bm{v})_{i,i}=\bm{v}_{i}$. Finally, for $\anymatrix,\othermatrix\in\mathbb{R}^{a\times b}$, we have:
\begin{equation*}
\ve(\anymatrix\odot \othermatrix)=\textnormal{diag}(\ve(\anymatrix))\ve(\othermatrix)\,.
\end{equation*}

By multiplying both sides of Eq.~(\ref{shiftedkron}) with $(\anymatrix\otimes\othermatrix
+\regparam\idmatrix)^{-1}$ from the left, we get
{\setlength\arraycolsep{2pt}
\begin{eqnarray}
\nonumber
\dualparvect
&=&(\anymatrix\otimes\othermatrix
+\regparam\idmatrix)^{-1}
\ve(\labelmatrix)\\
\nonumber
&=&((\evecmatrix\evalmatrix \evecmatrix^{-1})\otimes(\thirdmatrix\bm{\Sigma} \thirdmatrix^{-1})
+\regparam\idmatrix)^{-1}
\ve(\labelmatrix)\\
\nonumber
&=&((\evecmatrix\otimes \thirdmatrix)(\evalmatrix\otimes\bm{\Sigma})(\evecmatrix^{-1}\otimes \thirdmatrix^{-1})
+\regparam\idmatrix)^{-1}
\ve(\labelmatrix)\\
\label{eigenshift}
&=&(\evecmatrix\otimes \thirdmatrix)
(\evalmatrix\otimes\bm{\Sigma}+\regparam\idmatrix)^{-1}
(\evecmatrix^{-1}\otimes \thirdmatrix^{-1})
\ve(\labelmatrix)\\
\nonumber
&=&(\evecmatrix\otimes \thirdmatrix)
(\evalmatrix\otimes\bm{\Sigma}
+\regparam\idmatrix)^{-1}
\ve(\thirdmatrix^{-1}\labelmatrix \evecmatrix^{-\textnormal{T}})\\
\nonumber
&=&(\evecmatrix\otimes \thirdmatrix)
\ve(\bm{C}\odot \bm{E})\\
\label{finalshift}
&=&
\ve(\thirdmatrix(\bm{C}\odot \bm{E})\evecmatrix\transpose)\,,
\end{eqnarray}
}where
$\bm{E}=\thirdmatrix^{-1}\labelmatrix \evecmatrix^{-\textnormal{T}}$
and $\textrm{diag}(\ve(\bm{C}))=(\evalmatrix\otimes\bm{\Sigma}
+\regparam\idmatrix)^{-1}$.
In line (\ref{eigenshift}), we use (\ref{kroninv}) and therefore we can write $\regparam\idmatrix = \regparam(\evecmatrix\otimes \thirdmatrix)(\evecmatrix^{-1}\otimes \thirdmatrix^{-1})$ after which we can add $\regparam$ directly to the eigenvalues $\evalmatrix\otimes\bm{\Sigma}$ of $\anymatrix\otimes\othermatrix$.
The eigen decompositions of $\anymatrix$ and $\othermatrix$ as well as all matrix multiplications in (\ref{finalshift}) can be computed in $O(\nodecount^3)$ time.
\end{proof}

\begin{corollary}\label{regrcoro}
A minimizer of (\ref{rlsregr}) can be computed in $O(\nodecount^3)$ time.
\end{corollary}
\begin{proof}
Since the kernel matrix $\kernelm$ is symmetric and positive semi-definite, it is diagonalizable and it has nonnegative eigenvalues. This ensures that the matrix $\kernelm\otimes\kernelm+\regparam\idmatrix$ has strictly positive eigenvalues and therefore its inverse exists. Consequently, the claim follows directly from Proposition~\ref{shiftedlemma}, which can be observed by substituting $\kernelm$ for both $\anymatrix$ and $\othermatrix$.
\end{proof}


We continue by considering the use of the symmetric and reciprocal Kronecker kernels and show that, with those, the dual solution can be obtained as easily as with the ordinary Kronecker kernel. We first present and prove the following two inversion identities:
\begin{lemma}\label{identitylemma}
Let $\overline{\othermatrix}=\othermatrix\otimes\othermatrix$ for some square matrix $\othermatrix$. Then,
{\setlength\arraycolsep{2pt}
\begin{eqnarray}\label{identitysymm}
(\symm\overline{\othermatrix}\symm+\regparam\idmatrix)^{-1}
&=&\symm(\overline{\othermatrix}+\regparam\idmatrix)^{-1}\symm+\frac{1}{\regparam}\asymm \,,\\
\label{identityasymm}
(\asymm\overline{\othermatrix}\asymm+\regparam\idmatrix)^{-1}
&=&\asymm(\overline{\othermatrix}+\regparam\idmatrix)^{-1}\asymm+\frac{1}{\regparam}\symm \,,
\end{eqnarray}
}if the considered inverses exist.
\end{lemma}
\begin{proof}
For a start, we note certain directly verifiable properties of the symmetrizer and skew-symmetrizer matrices. Namely, the matrices $\symm$ and $\asymm$ are idempotent, that is,
\begin{equation*}
\symm\symm=\symm\textrm{ and }\asymm\asymm=\asymm \,.
\end{equation*}
Furthermore, $\symm$ and $\asymm$ are orthogonal to each other, that is,
\begin{equation}\label{symmorthog}
\symm\asymm=\asymm\symm=\bm{0} \,.
\end{equation}

We prove (\ref{identitysymm}) by multiplying $\symm\overline{\othermatrix}\symm+\regparam\idmatrix$ with its alleged inverse matrix and show that the result is the identity matrix:
{\setlength\arraycolsep{2pt}
\begin{eqnarray}
\nonumber
\lefteqn{(\symm\overline{\othermatrix}\symm+\regparam\idmatrix)
(\symm(\overline{\othermatrix}+\regparam\idmatrix)^{-1}\symm+\frac{1}{\regparam}\asymm)}\\
\label{pfsecondrow}
&=&\symm\overline{\othermatrix}\symm\symm(\overline{\othermatrix}+\regparam\idmatrix)^{-1}\symm+\frac{1}{\regparam}\symm\overline{\othermatrix}\symm\asymm\\
\nonumber
&&+\regparam\symm(\overline{\othermatrix}+\regparam\idmatrix)^{-1}\symm+\regparam\frac{1}{\regparam}\asymm\\
\label{pfthirdrow}
&=&\symm\overline{\othermatrix}(\overline{\othermatrix}+\regparam\idmatrix)^{-1}+\regparam\symm(\overline{\othermatrix}+\regparam\idmatrix)^{-1}+\asymm\\
\label{pffouthrow}
&=&\symm(\idmatrix-\regparam(\overline{\othermatrix}+\regparam\idmatrix)^{-1})+\regparam\symm(\overline{\othermatrix}+\regparam\idmatrix)^{-1}+\asymm\\
\nonumber
&=&\symm-\regparam\symm(\overline{\othermatrix}+\regparam\idmatrix)^{-1}+\regparam\symm(\overline{\othermatrix}+\regparam\idmatrix)^{-1}+\asymm\\
\nonumber
&=&\idmatrix\,.
\end{eqnarray}
}When going from (\ref{pfsecondrow}) to (\ref{pfthirdrow}), we use the fact that $\symm$ commutes with $(\overline{\othermatrix}+\regparam\idmatrix)^{-1}$, because it commutes with both $\overline{\othermatrix}$ and $\idmatrix$. Moreover, the second term of (\ref{pfsecondrow}) vanishes, because of the orthogonality of $\symm$ and $\asymm$ to each other. In (\ref{pffouthrow}) we have used the following inversion identity known in matrix calculus literature \citep{henderson1981dism}
\begin{equation*}
\overline{\othermatrix}(\overline{\othermatrix}+\idmatrix)^{-1}
=\idmatrix-(\overline{\othermatrix}+\idmatrix)^{-1}\,.
\end{equation*}
Identity (\ref{identityasymm}) can be proved analogously.
\end{proof}
These inversion identities indicate that we can invert a diagonally shifted symmetric or reciprocal Kronecker kernel matrix simply by modifying the inverse of a diagonally shifted ordinary Kronecker kernel matrix. This is an advantageous property, since the computational short-cuts provided by Proposition~\ref{shiftedlemma} ensure the fast inversion of the shifted ordinary Kronecker kernel matrices, and its results can thus be used to accelerate the computations for the symmetric and reciprocal cases too.

The next result uses the above inversion identities to show that, when learning symmetric or reciprocal relations with  kernel ridge regression \citep{saunders1998krr,Suykens2002,Shawetaylor2004,pahikkala2009preferences}, we do not explicitly have to use the symmetric and reciprocal Kronecker kernels. Instead, we can just use the ordinary Kronecker kernel to learn the desired model as long as we ensure that the symmetry or reciprocity is encoded in the labels.
\begin{proposition}
Using the symmetric Kronecker kernel for RLS regression with a label vector $\labelvector$ is equivalent to using an ordinary Kronecker kernel and a label vector $\symm\labelvector$. One can observe an analogous relationship between the reciprocal Kronecker kernel and a label vector $\asymm\labelvector$.
\end{proposition}
\begin{proof}
Let
{\setlength\arraycolsep{2pt}
\begin{eqnarray*}
\dualparvect&=&(\symm\overline{\kernelm}\symm+\regparam\idmatrix)^{-1}\labelvector\\
\bm{b}&=&(\overline{\kernelm}+\regparam\idmatrix)^{-1}\symm\labelvector
\end{eqnarray*}
}be solutions of (\ref{rlsregr}) with the symmetric Kronecker kernel and label vector $\labelvector$ and with the ordinary Kronecker kernel and label vector $\symm\labelvector$, respectively. Using identity (\ref{identitysymm}), we get
{\setlength\arraycolsep{2pt}
\begin{eqnarray*}
\dualparvect
&=&(\symm(\overline{\kernelm}+\regparam\idmatrix)^{-1}\symm+\frac{1}{\regparam}\asymm)\labelvector\\
&=&(\overline{\kernelm}+\regparam\idmatrix)^{-1}\symm\labelvector+\frac{1}{\regparam}\asymm\labelvector\,.
\end{eqnarray*}
}In the last equality, we again used the fact that $\symm$ commutes with $(\overline{\kernelm}+\regparam\idmatrix)^{-1}$, because it commutes with both $\overline{\kernelm}$ and $\idmatrix$. Let $(\node,\node')$ be a new couple of nodes for which we are supposed to do a prediction with a regressor determined by the coefficients~$\dualparvect$. Moreover, let $\bm{k_{\node}},\bm{k_{\node'}}\in\mathbb{R}^\nodecount$ denote, respectively, the base kernel $\kernelf^\phi$ evaluations of the nodes $\node$ and $\node'$ with the nodes in the training data. Then, $\bm{k_{v}}\otimes\bm{k_{v'}}\in\mathbb{R}^\edgecount$ contains the Kronecker kernel $\kernelf^\Phi_\otimes$ evaluations of the edge $(\node,\node')$ with all edges in the training data. Further, according to Proposition~\ref{symmasymmkmprop}, the corresponding vector of symmetric Kronecker kernel evaluations is $\symm(\bm{k_v}\otimes\bm{k_{v'}})$. Now, the prediction for the couple $(\node,\node')$ can be expressed as
{\setlength\arraycolsep{2pt}
\begin{eqnarray}
\nonumber
(\bm{k_v}\otimes\bm{k_{v'}})\transpose\symm\dualparvect
&=&(\bm{k_v}\otimes\bm{k_{v'}})\transpose\symm((\overline{\kernelm}+\regparam\idmatrix)^{-1}\symm\labelvector+\frac{1}{\regparam}\asymm\labelvector)\\
\nonumber
&=&(\bm{k_v}\otimes\bm{k_{v'}})\transpose\symm(\overline{\kernelm}+\regparam\idmatrix)^{-1}\symm\labelvector\\
\label{thisvanishes}
&&+\frac{1}{\regparam}(\bm{k_v}\otimes\bm{k_{v'}})\transpose\symm\asymm\labelvector\\
\nonumber
&=&(\bm{k_v}\otimes\bm{k_{v'}})\transpose\symm(\overline{\kernelm}+\regparam\idmatrix)^{-1}\symm\labelvector\\
\nonumber
&=&(\bm{k_v}\otimes\bm{k_{v'}})\transpose\symm\bm{b}\,,
\end{eqnarray}
}where term (\ref{thisvanishes}) vanishes due to (\ref{symmorthog}). The analogous result for the reciprocal Kronecker kernel can be shown in a similar way.
\end{proof}
As a consequence of this, we also have a computationally efficient method for RLS regression with symmetric and reciprocal Kronecker kernels. Encoding the properties into the label matrix ensures that the corresponding variations of the Kronecker kernels are implicitly used.

\subsection{Conditional Ranking with Symmetric and Reciprocal Kernels}\label{rankalgsect}

Now, we show how loss function (\ref{singlecentloss}) can be represented in matrix form. This representation is similar to the RankRLS loss introduced by \citet{pahikkala2009preferences}. Let
\begin{equation}\label{centeringmatrix}
\centerm^\nodedegree=\idmatrix-\frac{1}{l}\textbf{1}^{\nodedegree}{\textbf{1}^{\nodedegree}}\transpose\,,
\end{equation}
where $\nodedegree\in \mathbb{N}$, $\idmatrix$ is the $\nodedegree\times\nodedegree$-identity matrix, and $\textbf{1}^{\nodedegree}\in\mathbb{R}^\nodedegree$ is the vector of which every entry is equal to $1$, be the $\nodedegree\times\nodedegree$-centering matrix. The matrix $\centerm^\nodedegree$ is an idempotent matrix and multiplying it with a vector subtracts the mean of the vector entries from all elements of the vector. Moreover, the following equality can be shown
\begin{equation*}
\frac{1}{2l^2}\sum_{i,j=1}^{\nodedegree}(c_i-c_j)^2=\frac{1}{\nodedegree}\bm{c}\transpose\centerm^\nodedegree\bm{c} \,,
\end{equation*}
where $c_i$ are the entries of a vector $\bm{c}$. Now, let us consider the following quasi-diagonal matrix:
\begin{equation}\label{quasidiagmatrix}
\laplacian=
\left(
\begin{array}{ccc}
\centerm^{l_1}&&\\
&\ddots&\\
&&\centerm^{l_\nodecount}\\
\end{array}
\right)\,,
\end{equation}
where $\nodedegree_i$ is the number of edges starting from $\node_i$ for $i\in\{1,\ldots,\nodecount\}$. Again, given the assumption that the training data contains all possible edges between the nodes exactly once and hence $\nodedegree_i=\nodecount$ for all $1\leq i\leq\nodecount$, loss function (\ref{singlecentloss}) can be, up to multiplication with a constant, represented in matrix form as
\begin{eqnarray}\label{condrankmatrixform}
(\labelvector-\ol{\kernelm}\dualparvect)\transpose\laplacian(\labelvector-\ol{\kernelm}\dualparvect)\,,
\end{eqnarray}
provided that the entries of $\labelvector-\ol{\kernelm}\dualparvect$ are ordered in a way compatible with the entries of $\laplacian$, that is, the training edges are arranged according to their starting nodes.

Analogously to the regression case, the training phase corresponds to solving the following system of linear equations:
\begin{eqnarray}\label{linearsystem}
(\ol{\kernelm}\transpose\laplacian\ol{\kernelm}+\regparam\ol{\kernelm})\dualparvect
=\ol{\kernelm}\transpose\laplacian\labelvector\,.
\end{eqnarray}
If the ordinary Kronecker kernel is used, we get a result analogous to Corollary~\ref{regrcoro}.
\begin{corollary}\label{closedformprop}
A solution of (\ref{linearsystem}) can be computed in $O(\nodecount^3)$ time.
\end{corollary}
\begin{proof}
Given that $\nodedegree_i=\nodecount$ for all $1\leq i\leq\nodecount$ and that the ordinary Kronecker kernel is used, matrix (\ref{quasidiagmatrix}) can be written as $(\idmatrix\otimes\centerm^\nodecount)$ and the system of linear equations (\ref{linearsystem}) becomes:
\begin{eqnarray*}
\begin{array}{l}
((\kernelm\otimes\kernelm)(\idmatrix\otimes\centerm^\nodecount)(\kernelm\otimes\kernelm)+\regparam\kernelm\otimes\kernelm)\dualparvect\\
=(\kernelm\otimes\kernelm)(\idmatrix\otimes\centerm^\nodecount)\labelvector \,.
\end{array}
\end{eqnarray*}
While the kernel matrix $\kernelm\otimes\kernelm$ is not necessarily invertible, a solution can still be obtained from the following reduced form:
\begin{eqnarray*}
((\kernelm\otimes\kernelm)(\idmatrix\otimes\centerm^\nodecount)+\regparam\idmatrix)\dualparvect=(\idmatrix\otimes\centerm^\nodecount)\labelvector \,.
\end{eqnarray*}
This can, in turn, be rewritten as
\begin{eqnarray}\label{ranklinsystem}
(\kernelm\otimes\kernelm\centerm^\nodecount+\regparam\idmatrix)\dualparvect=(\idmatrix\otimes\centerm^\nodecount)\labelvector \,.
\end{eqnarray}
The matrix $\centerm^\nodecount$ is symmetric, and hence if $\kernelm$ is strictly positive definite,
the product $\kernelm\centerm^\nodecount$ is diagonalizable and has nonnegative eigenvalues (see e.g. \cite[p. 465]{horn1985matrixanalysis}). Therefore, (\ref{ranklinsystem}) is of the form which can be solved in $O(\nodecount^3)$ time due to Proposition~\ref{shiftedlemma}. The situation is more involved if $\kernelm$ is positive semi-definite. In this case, we can solve the so-called primal form with an empirical kernel map (see e.g. \citet{airola2011nystromcv}) instead of (\ref{ranklinsystem}) and again end up with a Kronecker system solvable in $O(\nodecount^3)$ time. We omit the details of this consideration due to its lengthiness and technicality.
\end{proof}

\subsection{Conjugate Gradient-Based Training Algorithms}\label{conjugatesection}

Interestingly, if we use the symmetric or reciprocal Kronecker kernel for conditional ranking, we do not have a similar efficient closed-form solution as those indicated by Corollaries~\ref{regrcoro}~and~\ref{closedformprop}. The same concerns both regression and ranking if the above assumption of the training data having every possible edge between all nodes encountered in the training data (i.e. $\nodedegree_i=\nodecount$ for all $1\leq i\leq\nodecount$) is dropped. Fortunately, we can still design algorithms that take advantage of the special structure of the kernel matrices and the loss function in speeding up the training process, while they are not as efficient as the above-described closed-form solutions.

Before proceeding, we introduce some extra notation. Let $\bookkeepmat\in\{0,1\}^{\edgecount\times\nodecount^2}$ be a bookkeeping matrix of the training data, that is, its rows and columns are indexed by the edges in the training data and the set of all possible pairs of nodes, respectively. Each row of $\bookkeepmat$ contains a single nonzero entry indicating to which pair of nodes the edge corresponds. This matrix covers both the situation in which some of the possible edges are not in the training data and the one in which there are several edges adjacent to the same nodes. Objective function (\ref{specialcaseobjfun}) can be written as
\begin{equation*}
\lossfunction(\labelvector,\bookkeepmat\ol{\kernelm}\dualparvect)
+\regparam\dualparvect\transpose\overline{\kernelm}\dualparvect
\end{equation*}
with the ordinary Kronecker kernel and analogously with the symmetric and reciprocal kernels. Note that the number of dual variables stored in vector $\dualparvect$ is still equal to $\nodecount^2$, while the number of labels in $\labelvector$ is equal to $\edgecount$. If an edge is not in the training data, the corresponding entry in $\dualparvect$ is zero, and if a particular edge occurs several times, the corresponding entry is the sum of the corresponding variables $a_\edge$ in representation~(\ref{eq:primalmodel}).
For the ranking loss, the system of linear equations to be solved becomes
\begin{eqnarray*}
(\ol{\kernelm}\bookkeepmat\transpose\laplacian\bookkeepmat\ol{\kernelm}+\regparam\ol{\kernelm})\dualparvect
=\ol{\kernelm}\bookkeepmat\transpose\laplacian\labelvector\,.
\end{eqnarray*}
If we use an identity matrix instead of $\laplacian$ in (\ref{condrankmatrixform}), the system corresponds to the regression loss.

To solve the above type of linear systems, we consider an approach based on conjugate gradient type of methods with early stopping regularization. The Kronecker product $(\kernelm\otimes\kernelm)\anyvector$ can be written as $\ve(\kernelm\bm{V}\kernelm)$, where $\anyvector=\ve(\bm{V})\in\mathbb{R}^{\nodecount^2}$ and $\bm{V}\in\mathbb{R}^{\nodecount\times\nodecount}$. Computing this product is cubic in the number of nodes. Moreover, multiplying a vector with the matrices $\symm$ or $\asymm$ does not increase the computational complexity, because they contain only $O(\nodecount^2)$ nonzero elements. Similarly, the matrix $\bookkeepmat$ has only $\edgecount$ non-zero elements. Finally, we observe from (\ref{centeringmatrix}) and (\ref{quasidiagmatrix}) that the matrix $\laplacian$ can be written as $\laplacian=\idmatrix-\slrmatrix\slrmatrix\transpose$, where $\slrmatrix\in\mathbb{R}^{\edgecount\times\nodecount}$ is the following
quasi-diagonal matrix:
\begin{equation}\label{slrdef}
\slrmatrix=
\left(
\begin{array}{ccc}
\frac{1}{\sqrt{l_1}}\textbf{1}^{l_1}&&\\
&\ddots&\\
&&\frac{1}{\sqrt{l_\nodecount}}\textbf{1}^{l_\nodecount}\\
\end{array}
\right)\,.
\end{equation}
The matrices $\idmatrix$ and $\slrmatrix$ both have $O(\edgecount)$ nonzero entries, and hence multiplying a vector with the matrix $\laplacian$ can also be performed in $O(\edgecount)$ time.

Conjugate gradient methods require, in the worst case, $O(\nodecount^4)$ iterations in order to solve the system of linear equations (\ref{linearsystem}) under consideration. However, the number of iterations required in practice is a small constant, as we will show in the experiments. In addition, since using early stopping with gradient-based methods has a regularizing effect on the learning process (see e.g. \citet{engl1996reginvprob}), this approach can be used instead of or together with Tikhonov regularization.

\subsection{Theoretical Considerations}

Next, we give theoretical insights to back the idea of using RankRLS-based learning methods instead of ordinary RLS regression. As observed in Section~\ref{rankalgsect}, the main difference between RankRLS and the ordinary RLS is that RankRLS enforces the learned models to be block-wise centered, that is, the aim is to learn models that, for each node $\node$, correctly predict the differences between the utility values of the edges $(\node,\node')$ and $(\node,\node'')$, rather than the utility values themselves. This is common for most of the pairwise learning to rank algorithms, since learning the individual utility values is, in ranking tasks, relevant only with relation to other utility values. Below, we consider whether the block-wise centering approach actually helps in achieving this aim. This is done via analyzing the regression performance of the utility value differences.

\newcommand{\objfunregr}{J}
\newcommand{\objfunrank}{F}
\newcommand{\objfunaltrank}{W}
\newcommand{\errbound}{B}

We start by considering the matrix forms of the objective functions of the ordinary RLS regression
\begin{equation}\label{origsetting}
\objfunregr(\dualparvect)=(\bm{y}-\overline{\kernelm}\dualparvect)\transpose(\bm{y}-\overline{\kernelm}\dualparvect)+\regparam\dualparvect\transpose\overline{\kernelm}\dualparvect
\end{equation}
and RankRLS for conditional ranking
\begin{equation}\label{origrankrls}
\objfunrank(\dualparvect)=(\labelvector-\overline{\kernelm}\dualparvect)\transpose\laplacian(\labelvector-\overline{\kernelm}\dualparvect)+\regparam\dualparvect\transpose\overline{\kernelm}\dualparvect\,,
\end{equation}
where $\laplacian\in\mathbb{R}^{\edgecount\times\edgecount}$ is a quasi-diagonal matrix whose diagonal blocks are $\nodecount\times\nodecount$-centering matrices. Here we make the further assumption that the label vector is block-wise centered, that is, $\labelvector=\laplacian\labelvector$. We are always free to make this assumption with conditional ranking tasks.


The following lemma indicates that we can consider the RankRLS problem as an ordinary RLS regression problem with a modified kernel.
\begin{lemma}\label{rankrlsasregressionlemma}
Objective functions (\ref{origrankrls}) and
\begin{equation}\label{rankrlsasakernel}
\objfunaltrank(\dualparvect)=(\labelvector-\laplacian\overline{\kernelm}\laplacian\dualparvect)\transpose(\labelvector-\laplacian\overline{\kernelm}\laplacian\dualparvect)+\regparam\dualparvect\transpose\laplacian\overline{\kernelm}\laplacian\dualparvect
\end{equation}
have a common minimizer.
\end{lemma}
\begin{proof}
By repeatedly applying the idempotence of $\laplacian$ and the inversion identities of \cite{henderson1981dism}, one of the solutions of (\ref{origrankrls}) can be written as
\begin{eqnarray*}
\dualparvect
&=&(\laplacian\overline{\kernelm}+\regparam\idmatrix)^{-1}\laplacian\labelvector\\
&=&(\laplacian\laplacian\overline{\kernelm}+\regparam\idmatrix)^{-1}\laplacian\labelvector\\
&=&\laplacian(\laplacian\overline{\kernelm}\laplacian+\regparam\idmatrix)^{-1}\labelvector\\
&=&\laplacian(\laplacian\overline{\kernelm}\laplacian\laplacian+\regparam\idmatrix)^{-1}\labelvector\\
&=&(\laplacian\laplacian\overline{\kernelm}\laplacian+\regparam\idmatrix)^{-1}\laplacian\labelvector\\
&=&(\laplacian\overline{\kernelm}\laplacian+\regparam\idmatrix)^{-1}\labelvector\,,
\end{eqnarray*}
which is also a minimizer of (\ref{rankrlsasakernel}). 
\end{proof}
This lemma provides us a different perspective on RankRLS. Namely, if we have a prior knowledge that the underlying regression function to be learned is block-wise centered (i.e. we have a conditional ranking task), this knowledge is simply encoded into a kernel function, just like we do with the knowledge about the reciprocity and symmetry.

In the literature, there are many results (see e.g. \citet{Vito2005inverse} and references therein) indicating that the expected prediction error of the regularized least-squared based kernel regression methods obey the following type of probabilistic upper bounds. For simplicity, we only consider the regression error. Namely, for any $0<\eta<1$, it holds that
\begin{equation}\label{boundeq}
P\left[I[\widehat{\predfun}_{\regparam,\tset}]-\textnormal{inf}_{\predfun\in\funspace_\kernelf}I[\predfun] \leq \errbound(\regparam,\kernelf,\eta)\right] \geq 1-\eta\,.
\end{equation}
where $P[\cdot]$ denotes the probability, $I[\cdot]$ is the expected prediction error, $\widehat{\predfun}_{\regparam,\tset}$ is the prediction function obtained via regularized risk minimization on a training set $\tset$ and a regularization parameter $\regparam$, $\funspace_\kernelf$ is the RKHS associated to the kernel $\kernelf$, and $\errbound(\regparam,\kernelf,\eta)$ is a complexity term depending on the kernel, the amount of regularization, and the confidence level $\eta$.

According to Lemma~\ref{rankrlsasregressionlemma}, if the underlying regression function $y$ is block-wise centered, which is the case in the conditional ranking tasks, we can consider learning with conditional RankRLS as performing regression with a block-wise centered kernel, and hence the behaviour of RLS regression and RankRLS can be compared with each other under the framework given in (\ref{boundeq}). When comparing the two kernels, we first have to pay attention to the corresponding RKHS constructions $\funspace_\kernelf$. The RKHS of the original kernel is more expressive than that of the block-wise centered kernel, because the former is able to express functions that are not block-wise centered while the latter can not. However, since we consider conditional ranking tasks, this extra expressiveness is of no help and the terms $\textnormal{inf}_{\predfun\in\funspace_\kernelf}I[\predfun]$ are equal for the two kernels.

Next, we focus our attention on the complexity term. A typical example of the term is the one proposed by \citet{Vito2005inverse}, which is proportional to $\kappa=\sup_{\edge}\kernelf(\edge,\edge)$. Now, the quantity $\kappa$ is lower for the block-wise centered kernel than for the original one, and hence the former has tighter error bounds than the latter. This, in turn, indicates that RankRLS is indeed a more promising approach for learning to predict the utility value differences than the ordinary RLS. It would be interesting to extend the analysis from the regression error to the pairwise ranking error itself but the analysis is far more challenging and it is considered as an open problem by \citet{Vito2005inverse}.

\section{Links with Existing Ranking Methods}
\label{adaptingsection}

Examining the pairwise loss (\ref{eq:combloss}) reveals that there exists a quite
straightforward mapping from the task of conditional ranking to that
of traditional ranking. Relation graph edges are in this mapping
explicitly used for training and prediction.
In recent years, several algorithms for learning to rank have
been proposed, which can be used for conditional ranking, by interpreting the conditioning node as a query
(see e.g.
\citet{joachims2002kddclickdata,Freund2003,pahikkala2009preferences}).
The main application has been in information retrieval, where the examples are joint feature representations of queries and documents,
and preferences are induced only between documents connected to the
same query. One of the earliest and most successful of these methods
is the ranking support vector machine RankSVM \citep{joachims2002kddclickdata}, which optimizes the
pairwise hinge loss. Even much more closely related is the ranking
regularized least-squares method RankRLS \citep{pahikkala2009preferences}, previously proposed by some
of the present authors. The method is based on minimizing the pairwise
regularized squared loss and becomes equivalent to the algorithms
proposed in this article, if it is trained directly on the relation
graph edges.

What this means in practice is that when the training relation graph is
sparse enough, say consisting of only a few thousand edges, existing methods
for learning to rank can be used to train conditional ranking models.
In fact this is how we perform the rock-paper-scissors experiments, as
discussed in Section~\ref{rpssection}. However, if the training graph is dense, existing
methods for learning to rank are of quite limited use.

Let us assume a training graph that has $\nodecount$ nodes. Furthermore,
we assume that most of the edges in the graph are connected, meaning
that the number of edges is of the order $\nodecount^2$. Using a
learning algorithm that explicitly calculates the kernel matrix
for the edges would thus need to construct and store
a $\nodecount^2 \times \nodecount^2$ matrix, which is intractable already
when $\nodecount$ is less than thousand. When the standard Kronecker
kernel is used together with a linear kernel for the nodes, primal
training algorithms
(see e.g. \citet{Joachims2006})
could be used without forming the kernel matrix.
Assuming on average $d$ non-zero features per node, this would
result in having to form a data matrix with $\nodecount^2d^2$ non-zero
entries. Again, this would be both memory-wise and computationally infeasible
for relatively modest values of $\nodecount$ and $d$.

Thus, building practical algorithms for solving the conditional ranking task
requires computational shortcuts to avoid the above-mentioned space and
time complexities. The methods presented in this article are based on such shortcuts, because queries and objects come from the same domain, resulting in a special structure of the Kronecker product kernel and a closed-form
solution for the minimizer of the pairwise regularized squared loss.

\section{Experiments}
\label{experimentsection}
In the experiments we consider conditional ranking tasks on synthetic and real-world data in various application domains, illustrating different aspects of the generality of our approach. The first experiment considers a potential application in game playing, using the synthetic rock-paper-scissors data set, in which the underlying relation is both reciprocal and intransitive.
The task is to learn a model for ranking players according to their likelihood of winning against any
other player on whom the ranking is conditioned.
The second experiment considers a potential application in information retrieval, using the 20-newsgroups data set. Here the task consists of ranking documents according
to their similarity to any other document, on which the ranking is conditioned. The third experiment summarizes a potential application of identifying bacterial species in microbiology. The goal consists of retrieving a bipartite ranking for a given species, in which bacteria from the same species have to be ranked before bacteria from a different species.
On both the newsgroups and bacterial data we test the capability of the models
to generalize to such newsgroups or species that have not been observed during
training.

In all the experiments, we run both the conditional ranker that minimizes
the convex edgewise ranking loss approximation (\ref{singlecentloss}) and
the method that minimizes the regression loss (\ref{regrloss}) over the edges.
Furthermore, in the rock-paper-scissors experiment we also train a conditional ranker
with RankSVM.  For the 20-newsgroups and bacterial species data
this is not possible due to the large number of edges present in the relational
graph, resulting in too high memory requirements and computational costs for RankSVM
training to be practical. We use the Kronecker kernel
$\kernelf_{\otimes}^{\Phi}$ for edges in all the experiments.
We also test the
effects of enforcing domain knowledge by applying the reciprocal kernel
$\kernelf_{\otimes R}^{\Phi}$ in the rock-paper-scissors experiment, and applying the
symmetric kernel $\kernelf_{\otimes S}^{\Phi}$ in the 20-newsgroups
and bacterial data experiments.
The linear kernel is used for individual nodes
(thus, for $\kernelf^{\phi}$).
In all the experiments, performance is measured using the ranking loss~(\ref{eq:combloss}) on the test set.


We use a variety of approaches for minimizing the squared conditional ranking
and regression losses, depending on the characteristics of the task. All the
solvers based on optimizing the standard, or pairwise regularized least-squares
loss are from the RLScore software
package\footnote{Available
at \url{http://www.tucs.fi/RLScore}}. For the experiment where the training is performed
iteratively, we apply the biconjugate gradient stabilized method
(BGSM) \citep{Vandervorst1992}. The RankSVM based conditional ranker
baseline considered in the rock-paper-scissors experiment is trained with the
TreeRankSVM software \citep{airola2011ranksvm}.

\subsection{Game Playing: the Rock-Paper-Scissors Dataset}
\label{rpssection}

The synthetic benchmark data,
whose generation process is described in detail by \citet{Pahikkala2010},
consists of simulated games of the well-known game of rock-paper-scissors between pairs of players. The
training set contains the outcomes of 1000 games played between
100 players, the outcomes are labeled according to which of the
players won. The test set consists of another group of 100 players,
and for each pair of players the probability of the first player
winning against the second one.
Different players differ in how often they play each of the three
possible moves in the game. The data set can be considered as a
directed graph where players are nodes and edges played games,
the true underlying relation generating the data is in this case
reciprocal. Moreover, the relation is intransitive. It represents
the probability that one player wins against another player.
Thus, it is not meaningful to try to construct a global ranking of the
players. In contrast, conditional ranking is a sensible task, where players are ranked
according to their estimated probability of winning against a given
player.

We experiment with three different variations of the data set,
the $w1$, $w10$ and $w100$ sets. These data sets differ in how balanced the
strategies played by the players are. In w1 all the players have close
to equal probability of playing any of the three available moves, while
in $w100$ each of the players has a favorite strategy he/she will use
much more often than the other strategies. Both the training and test sets
in the three cases are generated one hundred times and the hundred ranking
results are averaged for each of the three cases and for every tested
learning method.

Since the training set consists of only one thousand games, it is
feasible to adapt existing ranking algorithm implementations for solving
the conditional ranking task. Each game is represented as two edges, labeled
as $+1$ if the edge starts from the winner, and as $-1$ if the
edge starts from the loser. Each node has only $3$ features, and
thus, the explicit feature representation where the Kronecker kernel
is used together with a linear kernel results in $9$ product features
for each edge. In addition, we generate an analogous feature representation
for the reciprocal Kronecker kernel. We use these generated feature
representations for the edges to train three algorithms.
RLS regresses directly the edge
scores, RankRLS minimizes pairwise regularized squared loss
on the edges, and RankSVM minimizes pairwise hinge loss on the edges.
For RankRLS and RankSVM, pairwise preferences are generated only between
edges starting from the same node.

In initial preliminary experiments we noticed that on this data set
regularization seemed to be harmful, with methods typically reaching
optimal performance for close to zero regularization parameter values.
Further, cross-validation as a parameter selection strategy appeared to
work very poorly, due to the small training set size and the large amount
of noise present in the training data. Thus, we performed the runs using
a fixed regularization parameter set to a close to zero value ($2^{-30}$).

\begin{table*}[t]
\begin{center}
\begin{small}
\begin{tabular}{cllllll}
\hline
 & c.reg & c.reg (r)& c.rank & c.rank (r) & RankSVM & RankSVM (r)\\
\hline
$w1$ & 0.4875 & 0.4868 & 0.4876 & 0.4880 & 0.4987 & 0.4891\\
$w10$ & 0.04172 & 0.04145 & 0.04519 & 0.04291 & 0.04535 & 0.04116\\
$w100$ & 0.001380 & 0.001366 & 0.001424 & 0.001354 & 0.006997 & 0.005824\\
\hline
\end{tabular}
\end{small}
\end{center}
\caption{Overview of the measured rank loss for rock-paper-scissors. The
 abbreviations c.reg and c.rank here refer to the RLS and RankRLS algorithm, respectively, and (r) refers to the use of a reciprocal Kronecker kernel instead of the ordinary Kronecker kernel.}
\label{fig:rpsresults}
\end{table*}

The results of the experiments for the fixed regularization parameter value are presented in Table~\ref{fig:rpsresults}.
Clearly, the methods are successful in learning conditional ranking models, and
the easier the problem is made, the better the performance is. For all the methods and data sets,
except for the conditional ranking method with $w1$ data, the pairwise ranking error is smaller when using the reciprocal
kernel. Thus enforcing prior knowledge about the properties of the true underlying relation appears to be beneficial.
On this data set, standard regression proves to be competitive with the pairwise ranking approaches. Similar results,
where regression approaches can yield an equally good, or even a lower ranking error than rank loss optimizers, are known
in the recent literature, see e.g. \citet{pahikkala2009preferences,Kotlowski2011}. Somewhat surprisingly, RankSVM loses to the other methods in all the experiments other than the $w10$ experiment with reciprocal kernel, with difference being
especially large in the $w100$ experiment. 

\begin{figure}
\label{fig:nosignal}
\centering
\begin{tabular}{cc}
\includegraphics[width=\linewidth]{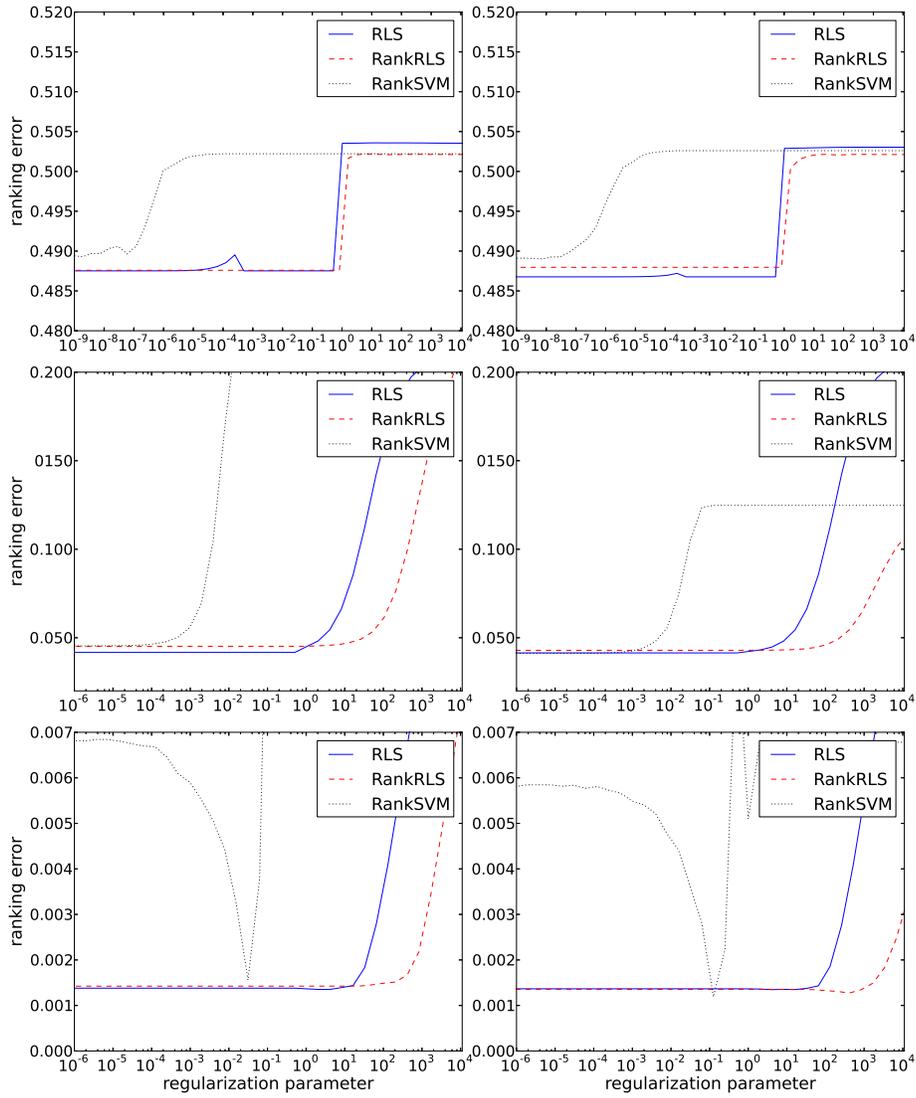}
\end{tabular}
\caption{Rock-paper-scissors data. Ranking test error as a function of
regularization parameter for the tested methods. Vertically: w1 (up), w10 (middle), w100 (bottom).
Horizontally: standard Kronecker kernel (left), reciprocal kernel (right).}
\label{fig:rpscurves}
\end{figure}

In order to have a more comprehensive view of the differences between the RLS, RankRLS and RankSVM results, we plotted the average test performance for the methods over the 100 repetitions of the experiments, for varying regularization parameter choices. The results are presented in Figure~\ref{fig:rpscurves}. For $w1$ and $w10$ data sets all the methods share a similar behaviour. The optimal ranking error can be reached for a range of smaller parameter values, until a point is reached where the error starts increasing. However, on $w100$ data sets RankSVM has quite a different type of behaviour\footnote{In order to ascertain that the difference was not simply caused by problems in the implementation or the underlying optimization library, we checked our results against those of the $\mathrm{SVM}^{rank}$ implementation available at \url{http://www.cs.cornell.edu/People/tj/svm_light/svm_rank.html}}. On this data, RankSVM can reach as good as, or even better performance than RLS or RankRLS, but only for a very narrow range of parameter values. Thus, for this data prior knowledge about the suitable parameter value would be needed in order to make RankSVM work, whereas the other approaches are more robust as long as the parameter is set to a fairly small value.

In conclusion, we have shown in this section
that highly intransitive relations can be modeled and successfully learned
in the conditional ranking setting. Moreover, we have shown that when
the relation graph of the training set is sparse enough, existing ranking
algorithms can be applied by explicitly using  the edges of the graph
as training examples. Further, the methods benefit from the use of the
reciprocal Kronecker kernel instead of the ordinary Kronecker kernel.
Finally, for this dataset it appears that a regression-based approach performs
as good as the pairwise ranking methods.

\begin{table*}[t]
\begin{center}
\begin{small}
\begin{tabular}{cllll}
	\hline
	 & Newsgr. 1 & Newsgr. 2 &  Bacterial 1 & Bacterial 2\\
	\hline
	c. rank & $0.2562$& $0.2895$&  0.1082 & 0.07631\\
	c. reg & $0.3685$ & $0.3967$ &  0.1084 & 0.07762\\
	\hline
\end{tabular}
\end{small}
\end{center}
\caption{Overview of the measured rank loss for the 20-Newsgroups and the bacterial
species ranking tasks in the large-scale experiments, where c.rank and c.reg
are trained using the closed-form solutions.}
\label{table:results}
\end{table*}

\subsection{Document Retrieval: the 20-Newsgroups Dataset}

In the second set of experiments we aim to learn to rank newsgroup documents
according to their similarity with respect to a document the ranking is
conditioned on. We use the publicly available 20-newsgroups data set\footnote{Available at: \url{http://people.csail.mit.edu/jrennie/20Newsgroups/}}
for the experiments. The data set consists of documents from 20 newsgroups, each
containing approximately 1000 documents, the document features are word frequencies.
Some of the newsgroups are considered
to have similar topics, such as the rec.sport.baseball, and
rec.sport.hockey newsgroups, which both contain messages about sports.
We define a three-level conditional ranking task. Given
a document, documents from the same newsgroup should be ranked the highest,
documents from similar newsgroups next, and documents from unrelated
newsgroups last. Thus, we aim to learn the conditional ranking model
from an undirected graph, and the underlying similarity relation is
a symmetric relation.
The setup is similar to that of \citet{Agarwal2006},
the difference is that we aim to learn a model for conditional
ranking instead of just ranking documents against a fixed newsgroup.

Since the training relation graph is complete, the number of edges
grows quadratically with the number of nodes. For $5000$ training nodes,
as considered in one of the experiments, this results already in a graph
of approximately 25 million edges. Thus, unlike in the previous
rock-paper-scissors experiment, training a ranking algorithm directly on the edges of the graph is no longer feasible. Instead, we solve the closed-form presented in Proposition~\ref{closedformprop}.
At the end of this section we also
present experimental results for the iterative conjugate gradient method, as
this allows us to examine the effects of early stopping, and enforcing symmetry on the prediction function.

In the first two experiments, where the closed-form solution is applied,
we assume a setting where the set of available newsgroups is not static,
but rather over time old newsgroups may wither
and die out, or new groups may be added. Thus, we cannot assume, when
seeing new examples, that we have seen documents from the same newsgroup
already when training our model.
We simulate this by selecting different newsgroups for testing than for
training. We form two disjoint sets of newsgroups. Set~1 contains
the messages from the newsgroups rec.autos, rec.sport.baseball,
comp.sys.ibm.pc.hardware and comp.windows.x, while set~2 contains the
messages from the newsgroups rec.motorcycles, rec.sport.hockey,
comp.graphics, comp.os.ms-windows.misc and comp.sys.mac.hardware.
Thus the graph formed by set~1 consists of approximately 4000
nodes, while the graph formed by set~2 contains
approximately 5000 nodes.
In the first experiment, set~1 is used for training and set~2 for
testing. In the second experiment, set~2 is used for training
and set~1 for testing. The regularization parameter is selected
by using half of the training newsgroups as a holdout set against
which the parameters are tested. When training the final model all
the training data is re-assembled.

The results for the closed-form solution experiments are presented in
Table~\ref{table:results}.
Both methods are successful in learning a conditional ranking model that
generalizes to new newsgroups which were not seen during the training
phase.
The method optimizing a ranking-based loss over the pairs
greatly outperforms the one regressing the values for the relations.




\begin{figure*}
\begin{center}
\includegraphics[width=0.8\linewidth]{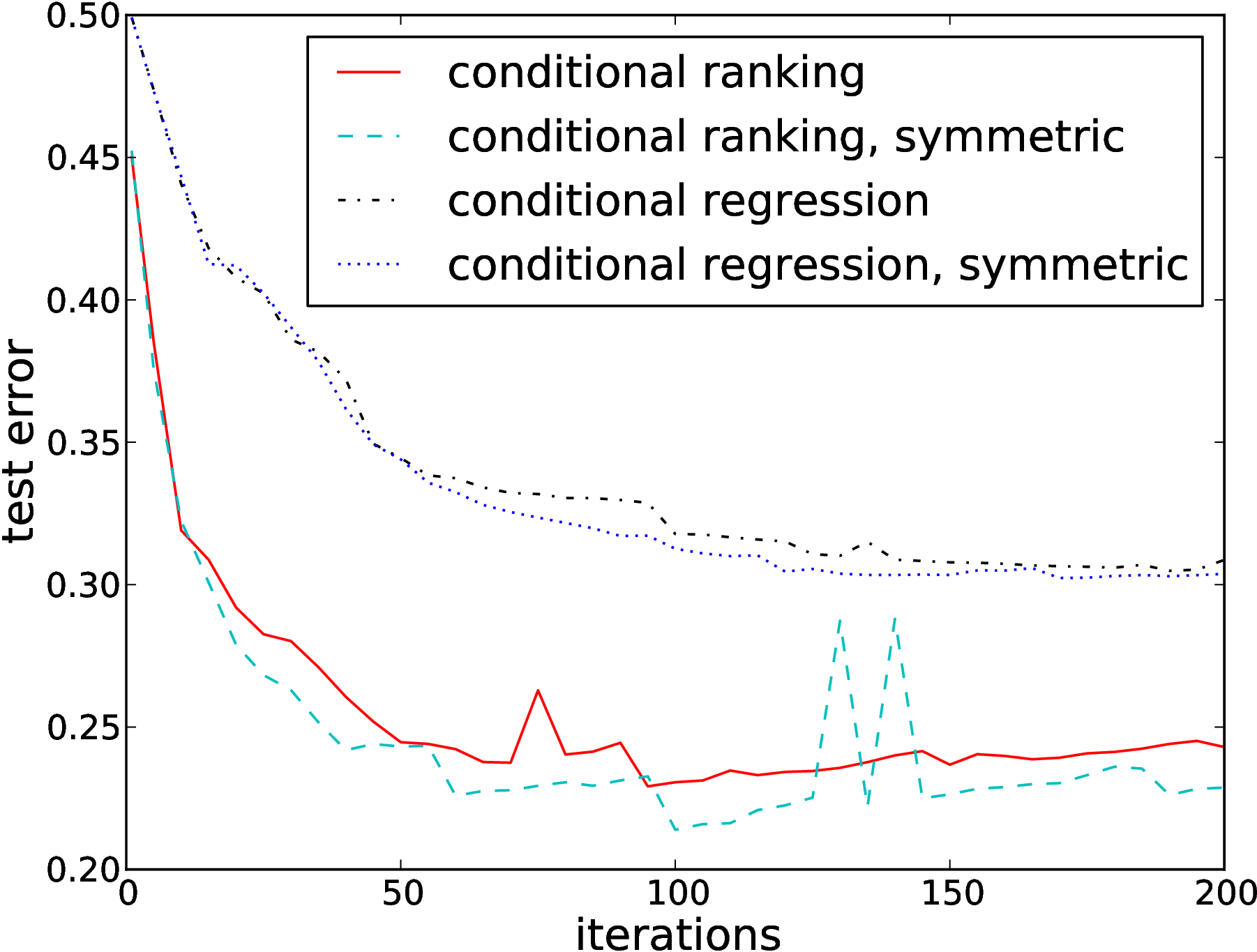}
\end{center}
\caption{Experimental results for the 20-Newsgroups data in the small-scale
experiment, in which all four models are learned using conjugate gradient
descent algorithms.}
\label{fig:curves}
\end{figure*}

Finally, we investigate whether enforcing the prior knowledge about the underlying
relation being symmetric is beneficial. In this final experiment we use
the iterative BGSM method, as it is compatible with the symmetric Kronecker kernel,
unlike the solution of Proposition~\ref{closedformprop}. The change
in setup results in an increased computational cost, since each iteration
of the BGSM method costs as much as using Proposition~\ref{closedformprop} to calculate
the solution. Therefore, we simplify the previous experimental setup by sampling a training
set of $1000$ nodes, and a test set of $500$ nodes from $4$ newsgroups.
The task is now easier than before, since the training and test sets
have the same distribution. All the methods are trained for $200$ iterations, and the test error
is plotted. We do not apply any regularization, but rather rely on the regularizing effect of early
stopping, as discussed in Section~\ref{algorithmsection}.

Figure~\ref{fig:curves} contains the performance curves. Again, we see that the
pairwise ranking loss quite clearly outperforms the regression loss. Using prior knowledge about the learned relation by enforcing symmetry leads to increased performance,
most notably for the ranking loss. The error rate curves are not monotonically
decreasing, but rather on some iterations the error may momentarily rise
sharply. This is due to the behaviour of the conjugate gradient optimization
scheme, which sometimes takes steps that lead further away from the optimal
solution. The performance curves
flatten out within the $200$ iterations, demonstrating the feasibility of
early stopping.

In conclusion, we have demonstrated various characteristics of our approach in the newsgroups experiments.
We showed that the introduced methods scale to training graphs that consist of tens of
millions of edges, each having a high-dimensional feature representation. We also showed the
generality of our approach, as it is possible to learn conditional ranking models even when the
test newsgroups are not represented in the training data, as long as data from similar newsgroups is available.
Unlike the earlier experiments on the rock-paper-scissors data, the pairwise loss yields a dramatic
improvement in performance compared to a regression-based loss. Finally,
enforcing prior knowledge about the type of the underlying relation with kernels was shown to be advantageous.

\subsection{Microbiology: Ranking Bacterial Species}

We also illustrate the potential of conditional ranking for multi-class
classification problems with a huge number of classes. For such problems it often happens that many classes are not represented in the training dataset, simply because no observations of these classes are known at the moment that the training dataset is constructed and the predictive model is learned. It speaks for itself that existing multi-class classification methods cannot make any correct predictions for observations of these classes, which might occur in the test set.

However, by reformulating the problem as a conditional ranking task, one is still capable of extracting some useful information for these classes during the test phase. The conditional ranking algorithms that we introduced in this article have the ability to condition a ranking on a target object that is unknown during the training phase. In a multi-class classification setting, we can condition the ranking on objects of classes that are not present in the training data. To this end, we consider bacterial species identification in microbiology.

In this application domain, one normally defines a multi-class classification problem with a huge number of classes as identifying bacterial species, given their fatty acid methyl ester (FAME) profile as input for the model \citep{Slabbinck2010FAME,MacLeod2010}. Here we reformulate this task as a conditional ranking task. For a given target FAME profile of a bacteria that is not necessarily present in the training dataset, the algorithm should rank all remaining FAME profiles of the same species higher than FAME profiles of
other species. For the most challenging scenario, none of these FAME profiles appears in the training dataset.

As a result, the underlying relational graph consists of two types of edges,
those connecting FAME profiles of identical species and those connecting FAME profiles of
different species. When conditioned on a single node, this setting realizes a
bipartite ranking problem, based on an underlying symmetric relation.

The data we used is described in more detail in
\citet{Slabbinck2010FAME}. Its original version consists of 955 FAME profiles, divided into 73
different classes that represent different bacterial species. A training set and
two separate test sets were formed as follows. The data points belonging
to the largest two classes were randomly divided between the training set, and
test set~1. Of the remaining smaller classes, 26 were included entirely in the
training set, and 27 were combined together to form test set~2. The difference
between the test sets was thus that FAME profiles from classes contained in test set~1 appear also in the training set, while this is not the
case for test set~2. The resulting set sizes were as follows. Training set: 473
nodes, test set~1: 308 nodes and test set~2: 174 nodes. Since the graphs are
fully connected, the number of edges grows quadratically with respect to
the number of nodes. The regularization parameter is chosen on a separate holdout set.

Due to the large number of edges, we train the rankers using the closed-form solution.
We also ran an experiment where we tested the effects of using
the symmetric Kronecker kernel, together with the iterative training algorithm.
In this experiment, using the symmetric Kronecker kernel leads to a very similar
performance as not using it, therefore we do not present these results separately.

Table~\ref{table:results} summarizes the resulting rank loss for the two different
test sets, obtained after training the conditional regression and ranking
methods using the closed-form solutions. Both methods are capable of training
accurate ranking models that can distinguish bacteria of the
same and different species groups, as the conditioning data points. Furthermore, comparing the results for test sets~1~and~2, we note that for this problem it is not
necessary to have bacteria from the same species present in both the test and training
sets, for the models to generalize. In fact, the test error on test set~2
is lower than the error on test set~1. The ranking-based loss function leads to
a slightly better test performance than regression.



\subsection{Bioinformatics: Functional ranking of enzymes}

As a last application we consider the problem of ranking a database of enzymes according to their catalytic similarity to a query protein. This catalytic similarity, which serves as the relation of interest, represents the relationship between enzymes w.r.t.\ their biological function. For newly discovered enzymes, this catalytic similarity is usually not known, so one can think of trying to predict it using machine learning algorithms and kernels that describe the structure-based or sequence-based similarity between enzymes. The Enzyme Commission (EC) functional classification is commonly used to subdivide enzymes into functional classes. EC numbers adopt a four-label hierarchical structure, representing different levels of catalytic detail.

We base the conditional rankings on the EC numbers of the enzymes, information which we assume to be available for the training data, but not at prediction time. This ground truth ranking can be deduced from the catalytic similarity (i.e.\, ground truth similarity) between the query and all database enzymes. To this end, we count the number of successive correspondences from left to right, starting from the first digit in the EC label of the query and the database enzymes, and stopping as soon as the first mismatch occurs. For example, an enzyme query with number EC~2.4.2.23 has a similarity of two with a database enzyme labeled EC~2.4.99.12, since both enzymes belong to the family of glycosyl transferases. The same query manifests a similarity of one with an enzyme labeled EC~2.8.2.23. Both enzymes are transferases in this case, but they show no further similarity in the chemistry of the reactions to be catalyzed. 

Our models were built and tested using a dataset of 1730 enzymes with known protein structures. All the enzyme structures had a resolution of at least 2.5~\AA, they had a binding site volume between 350 and 3500 $\text{\AA}^3$, and they were fully EC annotated. For evaluation purposes our database contained at least 20 observations for every EC number, leading to a total of 21 different EC numbers comprising members of all 6 top level codes. A heat map of the catalytic similarity of the enzymes is given in Figure~\ref{test_HM}. This catalytic similarity will be our relation of interest, constituting the output of the algorithm. As input we consider five state-of-the-art kernel matrices for enzymes, denoted cb (CavBase similarity), mcs (maximum common subgraph), lpcs (labeled point cloud superposition), wp (fingerprints) and wfp (weighted fingerprints). More details about the generation of these kernel matrices can be found in \cite{Stock2012monadic}. 

The dataset was randomized and split in four equal parts. Each part was withheld as a test set while the other three parts of the dataset were used for training and model selection. This process was repeated for each part so that every instance was used for training and testing (thus, four-fold outer cross-validation). In addition, a 10-fold inner cross validation loop was implemented for estimating the optimal regularization parameter $\lambda$, as recommended by~\cite{Varma2006}. The value of the hyperparameter was selected from a grid containing all the powers of 10 from $10^{-4}$ to $10^5$. The final model was trained using the whole training set and the median of the best hyperparameter values over the ten folds.

We benchmark our algorithms against an unsupervised procedure that is commonly used in bioinformatics for retrieval of enzymes. Given a specific enzyme query and one of the above similarity measures, a ranking is constructed by computing the similarity between the query and all other enzymes in the database. Enzymes having a high similarity to the query appear on top of the ranking, those exhibiting a low similarity end up at the bottom. More formally, let us represent the similarity between two enzymes by $K : \mathcal{V}^2  \rightarrow \mathbb{R}$, where $\mathcal{V}$ represents the set of all potential enzymes. Given the similarities $K(v,v')$ and $K(v,v'')$ we compose the ranking of $v'$ and $v''$  conditioned on the query $v$ as:
\begin{equation}
\label{eq:kernelranking}
v'  \succeq_v v'' \Leftrightarrow K(v,v') \geq K(v,v'')\,.
\end{equation}
This approach follows in principle the same methodology as a nearest neighbour classifier, but rather a ranking than a class label should be seen as the output of the algorithm.

\begin{table} [tb]
\begin{center}

\begin{footnotesize}
\begin{tabular}{llllll}
\hline
& cb & fp & wfp & mcs & lpcs\\
 \hline 
unsupervised & 0.0938 & 0.1185 & 0.1533 & 0.1077 & 0.1123 \\
c. reg & 0.0052 & 0.0050 & 0.0019 &
0.0054 & 0.0073 \\
c. rank & 0.0049 & 0.0050 & 0.0019 &
0.0056 & 0.0048 \\
 \end{tabular}
 \end{footnotesize}
\caption{A summary of the results obtained for the enzyme
ranking problem.}\label{results_unsup}
 \end{center}
 \end{table}

Table~\ref{results_unsup} gives a global summary of the results obtained for the different ranking approaches. All models score relatively well. One can observe that supervised ranking models outperform unsupervised ranking for all five kernels. Three important reasons can be put forward for explaining the improvement in performance. First of all, the traditional benefit of supervised learning plays an important role. One can expect that supervised ranking methods outperform unsupervised ranking methods, because they take ground-truth rankings into account during the training phase to steer towards retrieval of enzymes with a similar EC number. Conversely, unsupervised methods solely rely on the characterization of a meaningful similarity measure between enzymes, while ignoring EC numbers. 

Second, we also advocate that supervised ranking methods have the ability to preserve the hierarchical structure of EC numbers in their predicted rankings. Figure~\ref{test_HM} supports this claim. It summarizes the values used for ranking one fold of the test set obtained by the different models as well as the corresponding ground truth. So, for supervised ranking it visualizes the values $h(v,v')$, for unsupervised ranking it visualizes $K(v,v')$. Each row of the heatmap corresponds to one query. For the supervised models one notices a much better correspondence with the ground truth. Furthermore, the different levels of catalytic similarity can be better distinguished.

\begin{figure}[tb]
\includegraphics[width=0.24\textwidth]{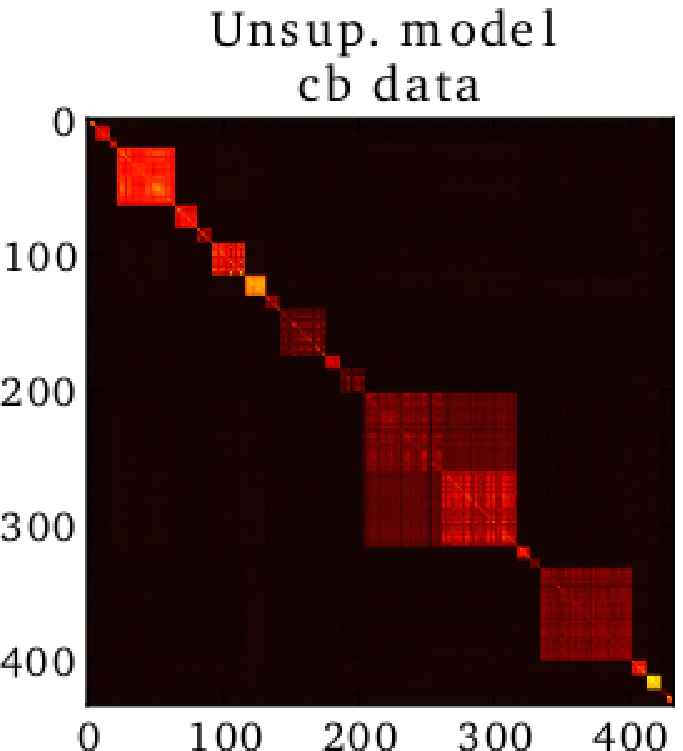}
\includegraphics[width=0.24\textwidth]{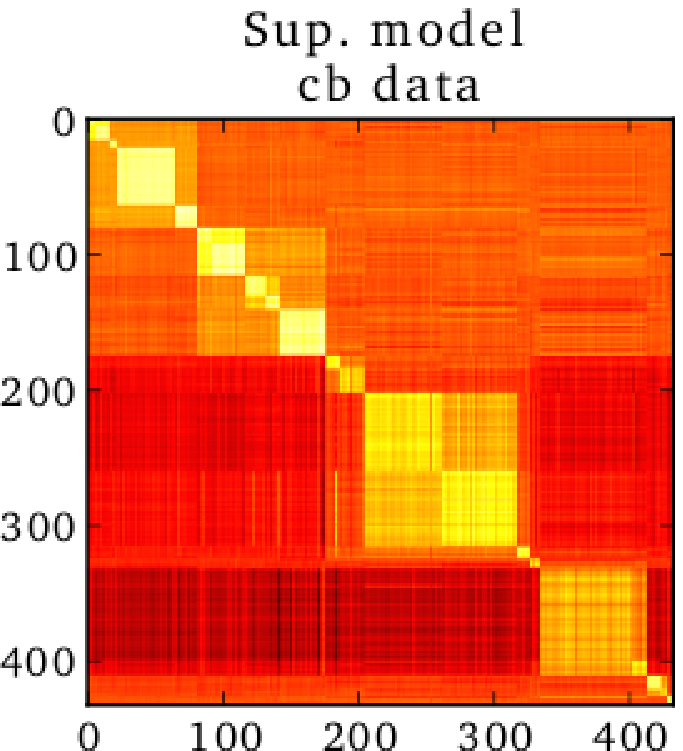}
\includegraphics[width=0.24\textwidth]{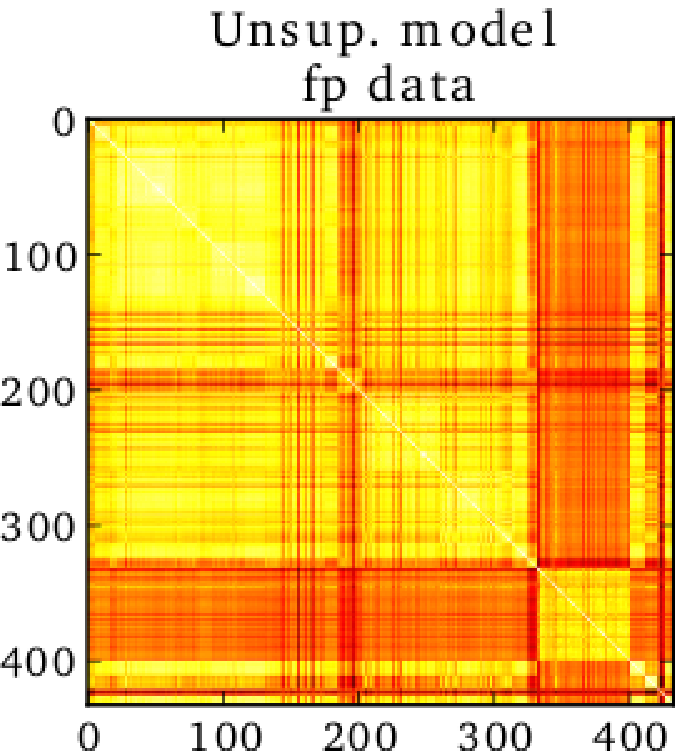}
\includegraphics[width=0.24\textwidth]{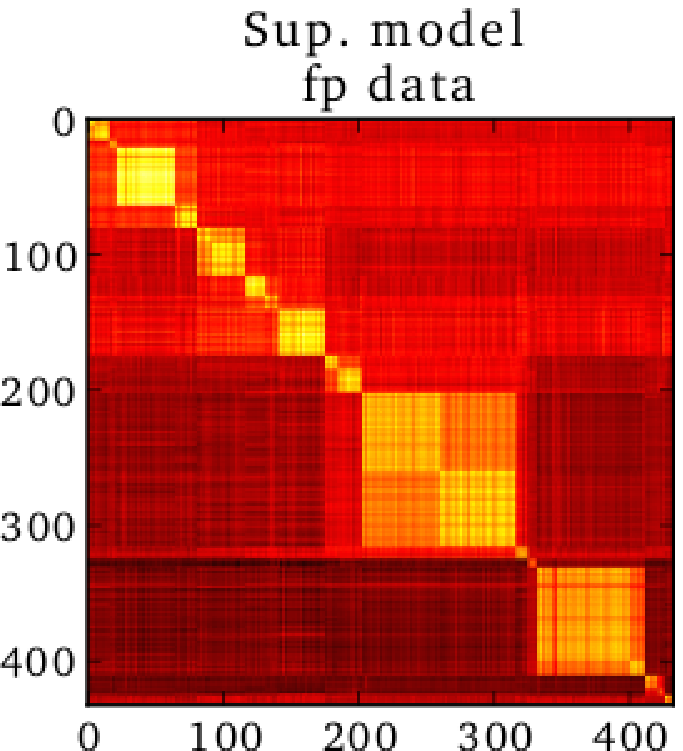}
\includegraphics[width=0.24\textwidth]{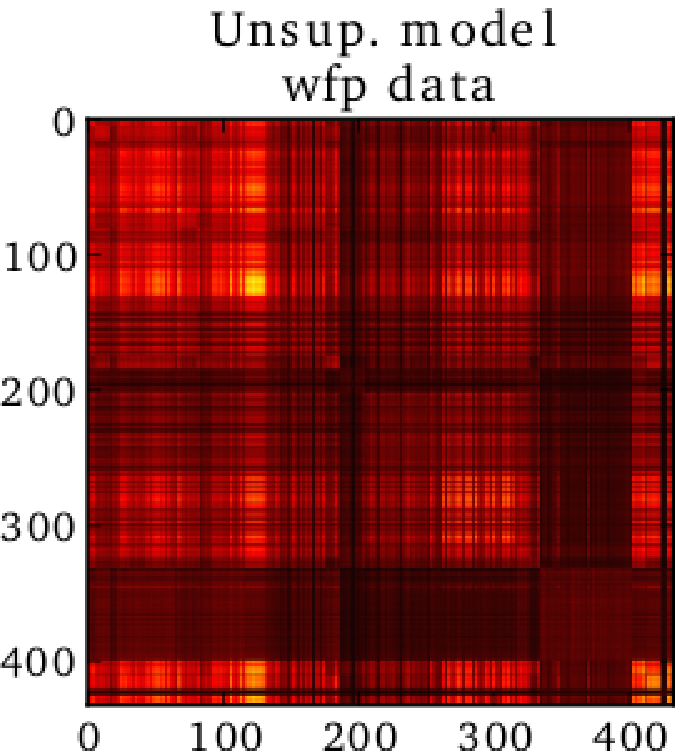}
\includegraphics[width=0.24\textwidth]{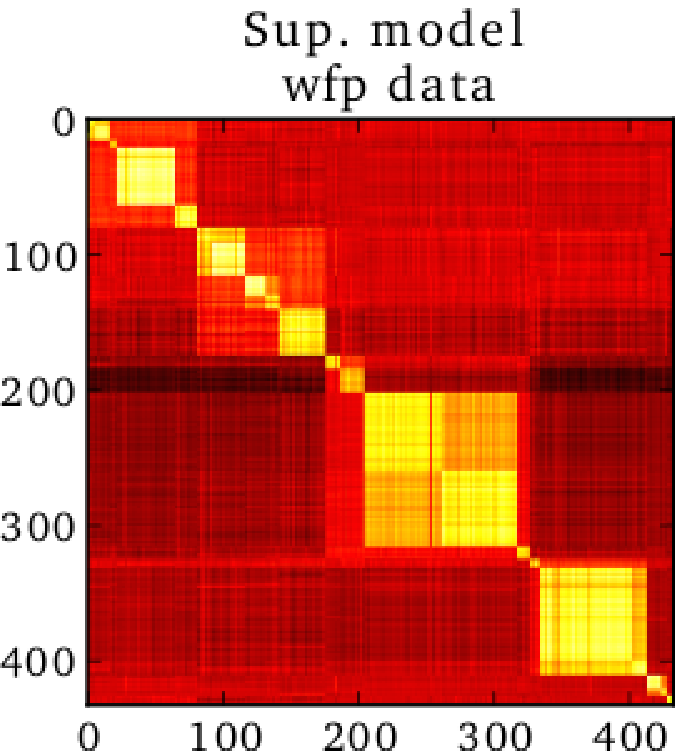}
\includegraphics[width=0.24\textwidth]{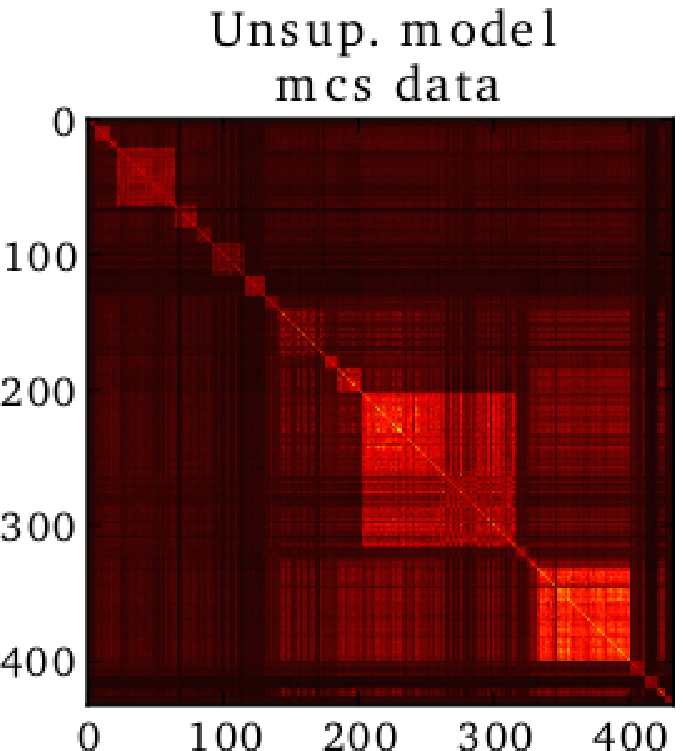}
\includegraphics[width=0.24\textwidth]{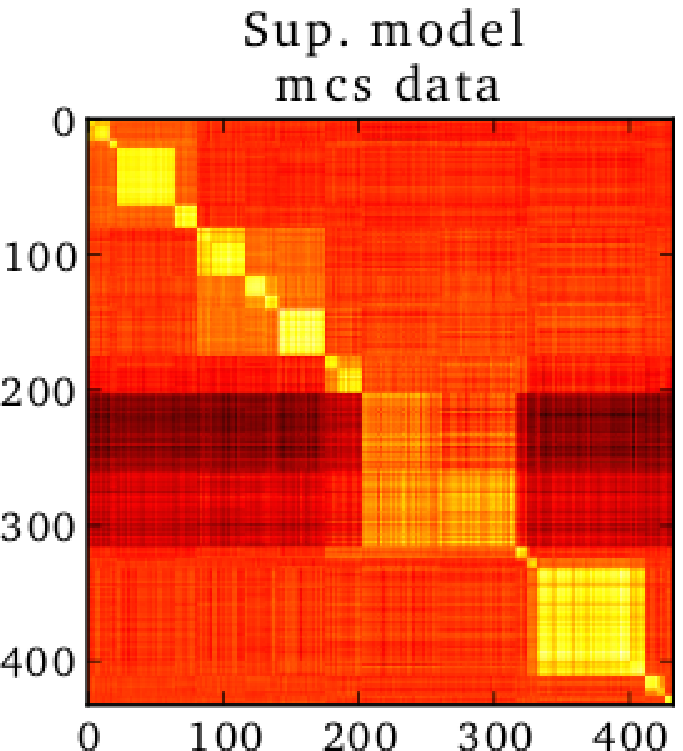}
\includegraphics[width=0.24\textwidth]{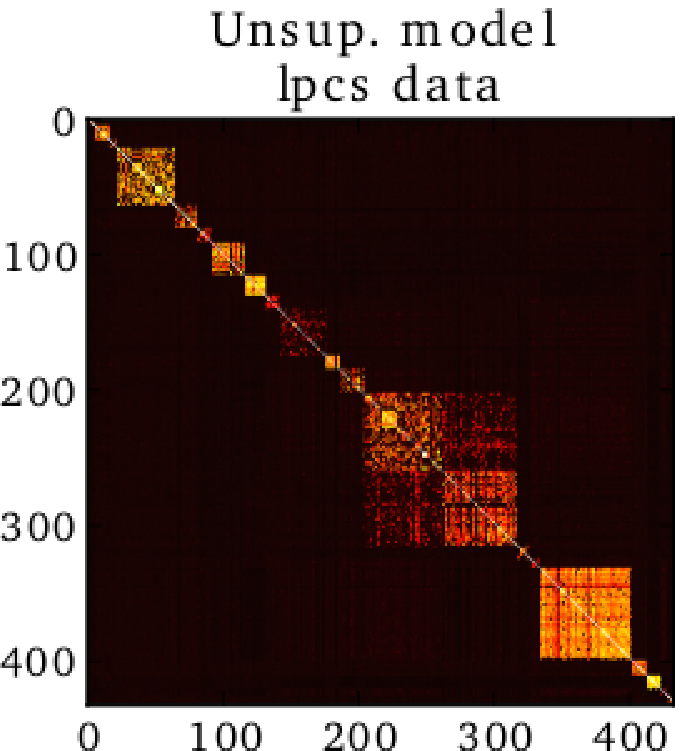}
\includegraphics[width=0.24\textwidth]{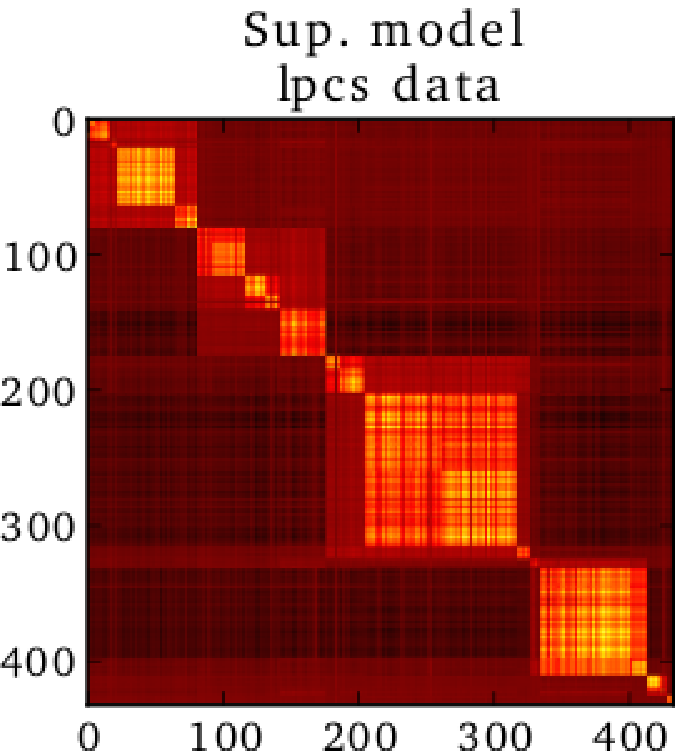}
\hfill\includegraphics[width=0.24\textwidth]{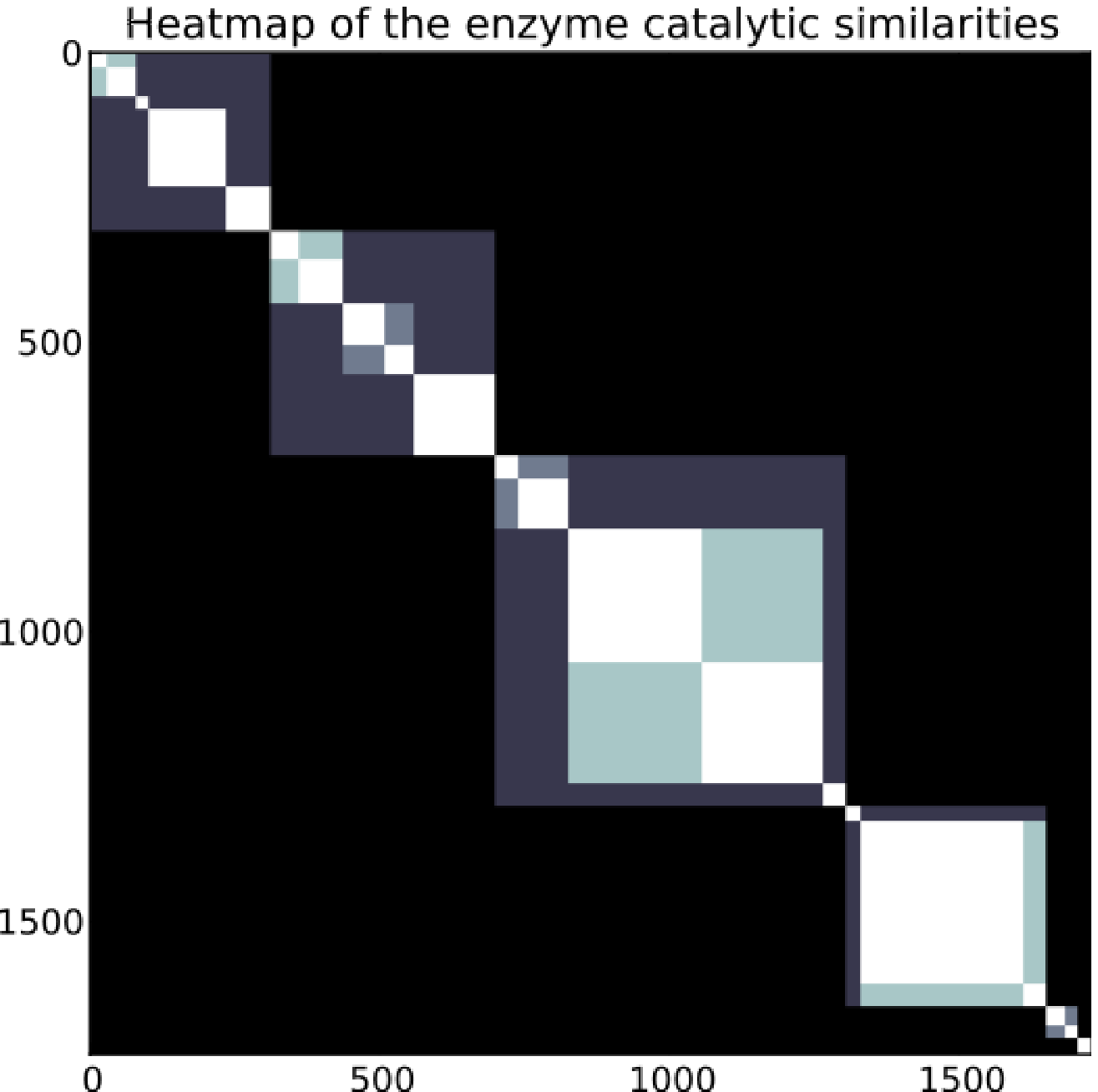}
\caption{Heatmaps of the values used for ranking the database during one fold in the testing phase. Each row of the heat map corresponds to one query. The corresponding ground truth is given in the lower right picture. The supervised model is trained by optimizing the pairwise ranking loss.}
\label{test_HM}
\end{figure}

A third reason for improvement by the supervised ranking method can be found in
the exploitation of dependencies between different similarity values. Roughly speaking, if one is interested in the similarity between enzymes $v$ and $u$, one can try to compute the similarity in a direct way, or derive it from the similarity with a third enzyme $z$. In the context of inferring protein-protein interaction and signal transduction networks, both methods are known as the direct and indirect approach, respectively~\citep{Vert2007,Geurts2007}. We argue that unsupervised ranking boils down to a direct approach, while supervised ranking should be interpreted as indirect. Especially when the kernel matrix contains noisy values, one can expect that the indirect approach allows for detecting the \emph{back bone} entries and correcting the noisy ones.

The results for the two supervised conditional ranking approaches are in many cases similar, with both models having same predictive performance on two of the kernels (fp and wfp). For one of the kernels (lpcs) ranking loss gives much better performance than the regression one, for another kernel (cb) ranking loss has a slight advantage, and in the remaining experiment (mcs) the regression approach performs slightly better. An appropriate choice of the node-level kernel proves to be much more important than the choice of the loss function, as the supervised models trained using the wfp kernel clearly outperform all other approaches.

\subsection{Runtime performance}

\begin{figure*}
\begin{center}
\includegraphics[width=0.6\linewidth]{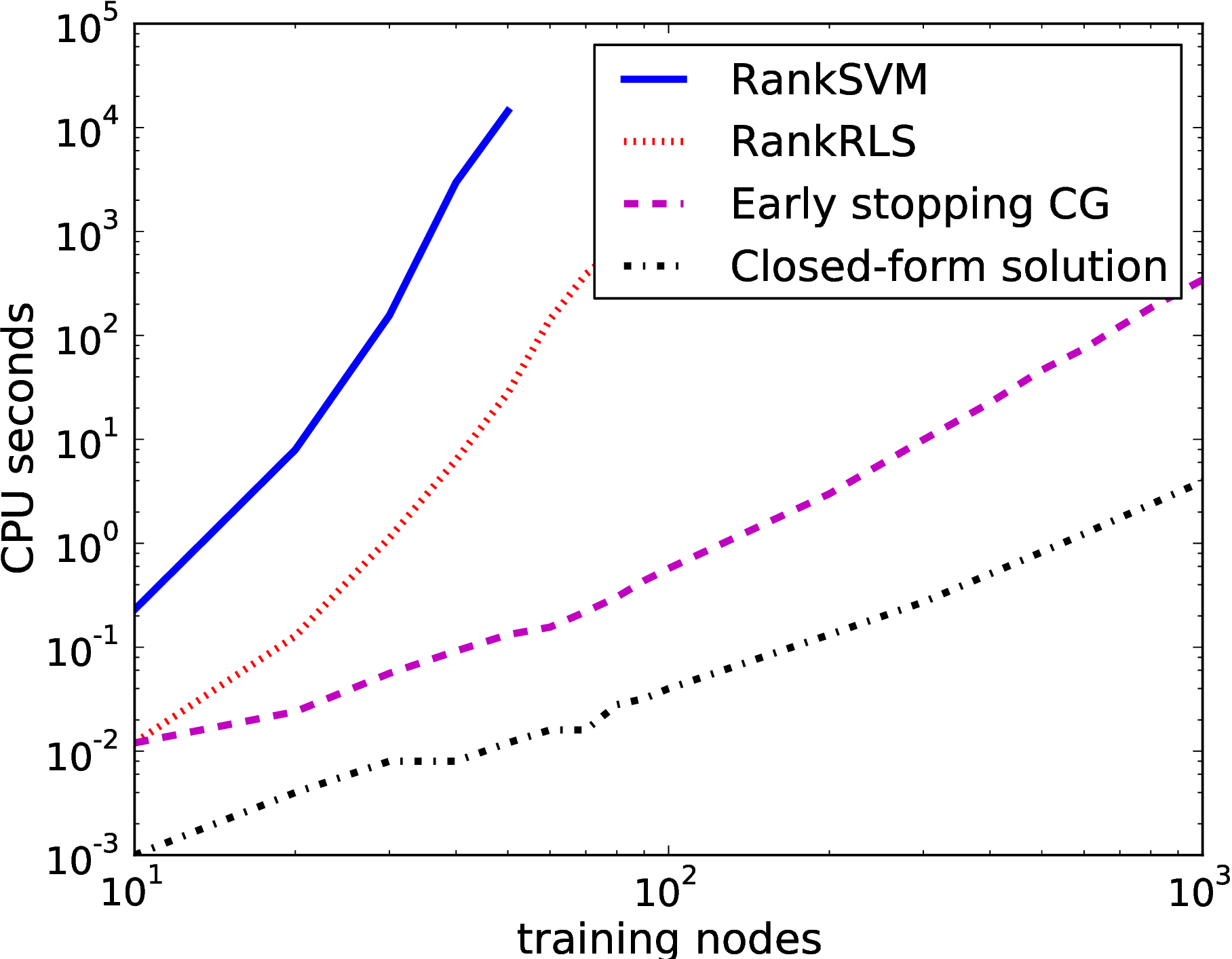}
\includegraphics[width=0.6\linewidth]{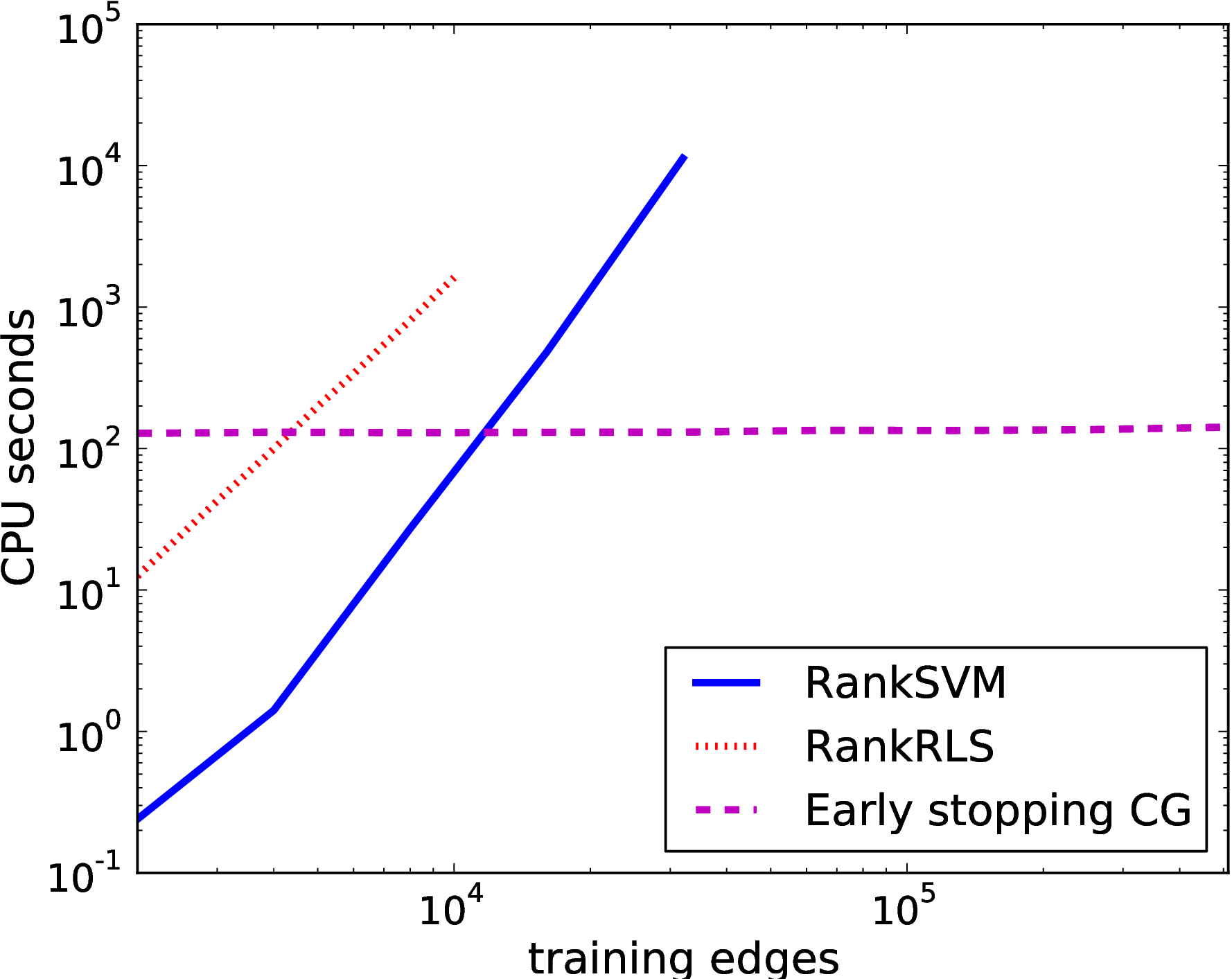}
\includegraphics[width=0.6\linewidth]{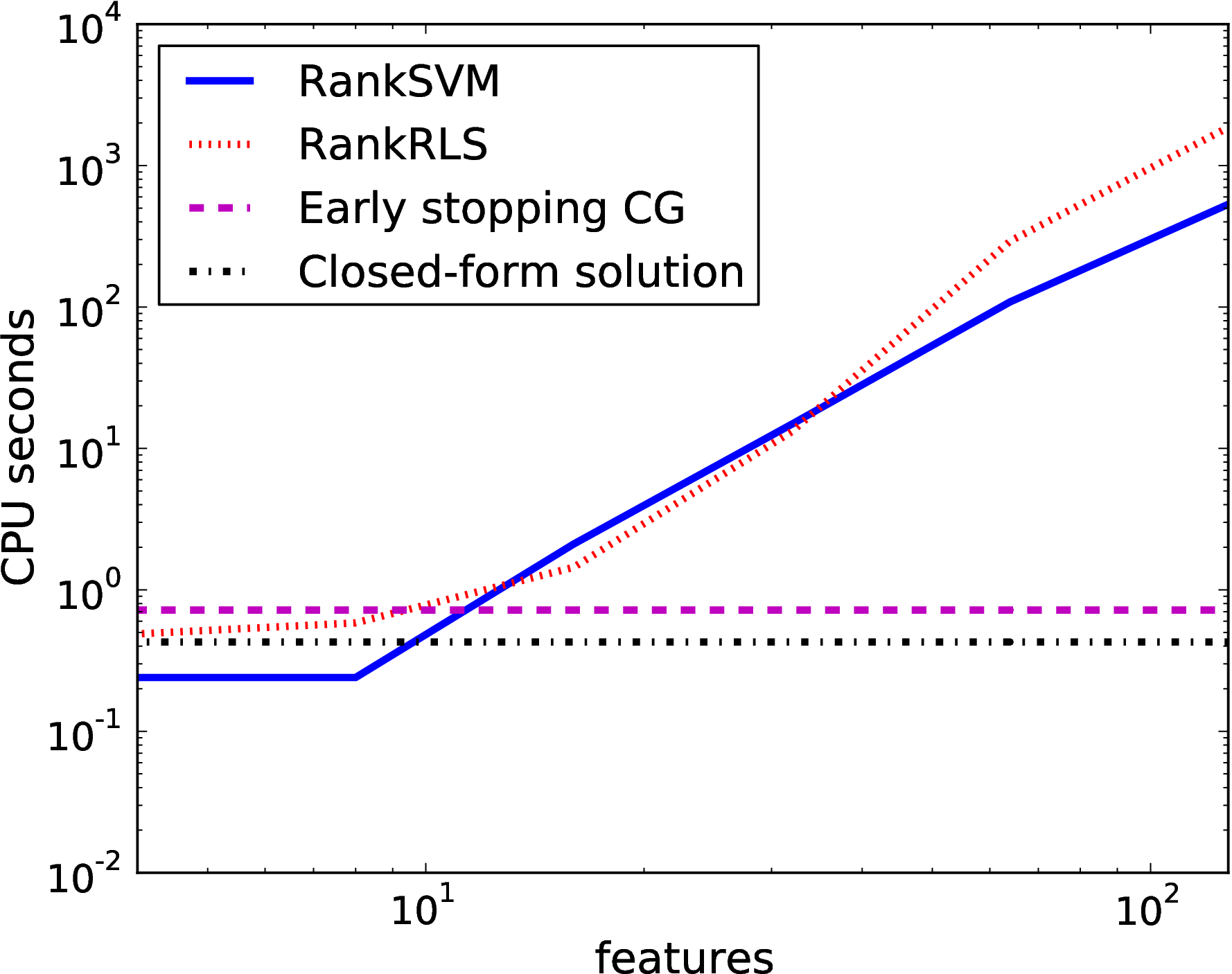}
\end{center}
\caption{Scalability experiments for training different algorithms on
a sample of the 20-Newsgroups data. We consider both a fully connected
training graph with varying amount of nodes (top) as well as a graph with 1000
nodes and varying degrees of sparsity (middle). Finally, we consider linear
solvers with a fully connected graph of 100 nodes and varying number of features
(bottom)}
\label{fig:scaling}
\end{figure*}

In the runtime experiment we compare the computational efficiency of the conditional ranking approaches considered in Section~\ref{algorithmsection}. We consider conjugate gradient training with early stopping and the closed-form solution, as well as two off-the-shelf ranking algorithms trained directly on the edges, namely RankRLS and RankSVM. For kernel RankSVM, we use the $\mathrm{SVM}^{light}$ package, implementing the Kronecker product kernel in the $kernel.h$ file. The linear RankSVM experiments are run using the $\mathrm{SVM}^{rank}$ package. The other methods are implemented in the RLScore software. All experiments are run on a single core with an Intel Core~i7-3770 processor, with $16$~GB memory available. The experiments are performed with regularization parameter $\lambda=1$, and a limit of $200$ iterations for the conjugate gradient method.

In the first two experiments, we consider the efficiency of the different
kernelized solvers on samples of the Reuters data. First, we measure the scalability
of the methods on a fully connected graph with a varying number of nodes.
Second, we fix the number of nodes to $1000$ and vary the number of edges.

The results are presented in Figure~\ref{fig:scaling} (top and middle). First,
let us consider the scaling behaviour on a fully connected graph. The kernel RankRLS
solver has cubic time complexity, training it on all the edges in the fully
connected training graph thus results in $O(\nodecount^6)$ time complexity. It can be observed that in
practice the approach does not scale beyond tens of nodes (thousands
of edges), meaning that it cannot be applied beyond small toy problems. The
RankSVM solver has even worse scaling. In contrast, the iterative training algorithm
(Early stopping CG) and the closed-form solution allow efficient scaling to
graphs with thousands of nodes, and hence millions of edges. While the iterative
training method and the closed-form share the same $O(\nodecount^3)$ asymptotic
behaviour, the closed-form solution allows an order of magnitude faster training, making
it the method of choice whenever applicable.

Next, we consider training the algorithms on sparse graphs. When the graph is
very sparse, meaning that there are only on average around ten or less edges for
each node, the RankSVM solver is the most efficient method to use. Once the
graph becomes denser, using the Kronecker product shortcuts becomes necessary.
Beyond $32000$ edges only the Early stopping CG method, whose iteration cost
does not depend on the number of edges, proves feasible to use.

Finally, we performed an experiment where we compare the proposed algorithms to
existing linear solvers, using low-dimensional data and the linear kernel. We
sampled $100$ data points from the UCI repository USPS data set, and
generated a fully connected label graph by assigning label $1$ to data point
pairs belonging to the same, and $0$ to pairs belonging to different classes. We
vary the dimensionality of the data by sampling the features. The linear solvers
are trained on explicitly computed Kronecker product features. As shown in
Figure~\ref{fig:scaling} (bottom), the RankSVM and RankRLS solvers are
feasible to use and even competitive if the number of features is very low
(e.g. 10 or less), as in this case the number of generated product features is
also low enough to allow for efficient training. As the number of features grows,
using basic RankSVM or RankRLS, however, becomes first inefficient and then
infeasible, we did not perform experiments for more than $128$ features since
soon after this point the data matrix no longer fits into memory. We also
performed experiments with $1000$ nodes, in this case linear RankRLS and RankSVM
did not scale beyond $20$ features.

The results further demonstrate our claims about the
scalability of the proposed algorithms to large dense graphs. Even with a
non-optimized high-level programming language implementation (Python), one can
handle training a kernel solver on million edges in a matter of minutes. On
very sparse graphs, or when applying linear models with low-dimensional data using
existing solvers may also prove feasible.

\section{Conclusion}

We presented a general framework for conditional ranking from
various types of relational data, where rankings can be conditioned on unseen
objects. We proposed efficient least-squares algorithms for optimizing
regression and  ranking-based loss functions, and presented generalization
bounds motivating the advantages of using the ranking based loss.
Symmetric or reciprocal relations can be treated as two important
special cases of the framework, we prove that such prior knowledge
can be enforced without having to sacrifice computational
efficiency. Experimental results on both synthetic and real-world
datasets confirm that the conditional ranking problem can be solved
efficiently and accurately, and that optimizing a ranking-based loss can be
beneficial, instead of aiming to predict the underlying relations directly.
Moreover, we also showed empirically that incorporating domain knowledge about
the underlying relations can boost the predictive performance.

Briefly summarized, we have discussed the following three approaches for solving conditional ranking problems:
\begin{itemize*}
\item off-the-shelf ranking algorithms can be used when they can be computationally afforded, i.e., when the number of edges in the training set is small;
\item the above-presented approach based on the conjugate gradient method with early stopping and taking advantage of the special matrix structures is recommended when using off-the-shelf methods becomes intractable;
\item the closed-form solution presented in Proposition~\ref{closedformprop} is
recommended in case the training graph is fully connected, since its
computational complexity is equivalent to that of a single iteration of the conjugate gradient method.
\end{itemize*}
Both the computational complexity analysis and the scalability experiments
demonstrate, that the introduced algorithms allow solving orders of magnitudes
larger conditional ranking problems than was previously possible with
existing ranking algorithms.

\section*{Acknowledgments}

We would like to thank the anonymous reviewers for their insightful comments.
T.P. and A.A. are both supported for this work by the Academy of Finland (grant 134020 and 128061, respectively). W.W. is supported by the Research Foundation of Flanders. A preliminary version of this work was presented at the European Conference on Machine Learning in 2010 \citep{pahikkala2010conditional}.

\bibliographystyle{abbrvnat}
\bibliography{myBibliography_reduced}
\end{document}